\documentclass[conference, letterpaper]{IEEEtran}
\usepackage[dvips]{graphicx}
\usepackage[cmex10]{amsmath}
\usepackage{amsthm}
\usepackage{bbm}
\usepackage[colorlinks=true,linkcolor=red,citecolor=blue]{hyperref}
\usepackage{cleveref}
\usepackage{booktabs}
\usepackage{url}
\usepackage{color}
\usepackage{bm}
\usepackage{amssymb}
\usepackage{mathtools}
\usepackage[utf8]{inputenc} 
\usepackage[T1]{fontenc}
\usepackage{url}
\usepackage{ifthen}
\usepackage{cite}

\newtheorem{theorem}{Theorem}
\newtheorem{definition}{Definition}
\newtheorem{remark}{Remark}
\newtheorem{claim}{Claim}
\newtheorem{lemma}{Lemma}
\newtheorem{assumption}{Assumption}

\newcommand{\rate}[1]{\alpha_{#1}}

\DeclareMathOperator{\MMSE}{MMSE}
\DeclareMathOperator{\MLE}{MLE}
\DeclareMathOperator{\UB}{UB}
\DeclareMathOperator{\LB}{LB}
\DeclareMathOperator{\AMP}{AMP}
\DeclareMathOperator{\G}{G}
\DeclareMathOperator{\F}{F}
\DeclareMathOperator{\I}{I}
\DeclareMathOperator{\SOS}{SOS}

\def\<{\langle}
\def\>{\rangle}

\def\E{{\mathbb E}} 

\def\reals{\mathbb{R}}

\def\bx{\mathbf{x}}

\def\bv{\mathbf{v}}

\def\bx{\mathbf{x}}

\def\bY{\mathbf{Y}}

\def\bZ{\mathbf{Z}}

\def\bU{\mathbf{U}}

\def\bV{\mathbf{V}}

\def\E{\mathbb{E}}
\def\normal{\mathcal{N}}
\def\F{{\sf F}}

\def\od{{^{\otimes d}}}

\def\P{\mathbb{P}}
\def\id{{\rm I}}

\def\ind{{\mathbb{I}}}

\def\E{{\mathbb E}} 

\def\id{{\rm I}}
\def\<{\langle}
\def\>{\rangle}

\def\ind{i_1 i_2\cdots i_d}
\def\indperm{\pi(i_1)\pi( i_2)\cdots \pi(i_h)}

\def\E{{\mathbb E}} 

\def\id{{\rm I}}
\def\<{\langle}
\def\>{\rangle}

\def\E{{\mathbb E}} 

\def\id{{\rm I}}
\def\<{\langle}
\def\>{\rangle}

\DeclareMathOperator*{\argmax}{argmax}

\newcommand{\Gaussian}{\mathcal{N}}
\renewcommand{\Pr}{\mathbb{P}}

\newcommand{\plant}{S_{\sf planted}}

\newcommand{\ra}[1]{\renewcommand{\arraystretch}{#1}}
\newcommand{\share}{r}

\crefname{appsec}{appendix}{appendices}

\definecolor{emphcol}{gray}{.75}

\def \LONG{}

\begin{document}

\title{Statistical and computational thresholds for the planted $k$-densest sub-hypergraph problem}
\author{\IEEEauthorblockN{Luca Corinzia, and
Paolo Penna}
\IEEEauthorblockA{
Department of Computer Science\\
ETH Zürich, Switzerland\\
Email: \{luca.corinzia,paolo.penna\}@inf.ethz.ch}
\and
\IEEEauthorblockN{Wojciech Szpankowski}
\IEEEauthorblockA{Department of Computer Science \\
Purdue University, USA\\
Email: szpan@purdue.edu}
\and 
\IEEEauthorblockN{Joachim M. Buhmann}
\IEEEauthorblockA{
Department of Computer Science\\
ETH Zürich, Switzerland\\
Email: jbuhmann@inf.ethz.ch}
}
\maketitle

\begin{abstract}
In this work, we consider the problem of recovery a planted $k$-densest sub-hypergraph on $d$-uniform hypergraphs. 
This fundamental problem appears in different contexts, e.g., community detection, average-case complexity, and neuroscience applications as a structural variant of tensor-PCA problem. We provide tight \emph{information-theoretic} upper and lower bounds for the exact recovery threshold by the maximum-likelihood estimator, as well as \emph{algorithmic} bounds based on approximate message passing algorithms.
The problem exhibits a typical statistical-to-computational gap observed in analogous sparse settings that widen with increasing sparsity of the problem. 
The bounds show that the signal structure impacts the location of the statistical and computational phase transition that the known existing bounds for the tensor-PCA model do not capture. This effect is due to the generic planted signal prior that this latter model addresses.
\end{abstract}
\ifdefined\LONG \else 
\textit{A full version of this paper is accessible at:}
\url{https://arxiv.org/pdf/2011.11500.pdf}
\fi

\section{Introduction}
High dimensional inference problems play a key role in recent machine learning and data analysis applications. Typical scenarios exhibit problem dimensions comparable to the sample size, hence precluding effective estimation with no further structure imposed on the underlying signal, such as low-rank or sparsity. 
Most of these problems, as sparse mean estimation, compress sensing, low-rank matrix estimation, and planted clique estimation can be regarded as the general problem of recovery a planted signal when the available observations are perturbed by noise. 

In this work we study the problem of recovery a \emph{planted} $k$-densest sub-hypergraph on $d$-uniform hypergraph over $p$ nodes. This problem, considered in \cite{corinzia2019}, is closely related to several community detection models and to the tensor-PCA. It addresses high order interactions described by hyperedges between nodes, that arise naturally in several applications, including brain region modeling \cite{gu2017functional, wang2012naive, zu2016identifying} and memory \cite{legenstein2018long}, computer vision applications \cite{jolion2012graph}, and network analysis \cite{Benson163,grilli2017higher,ke2019community}.

The planted solution consists of a subset of $k$ randomly chosen nodes out of $p$. The weights of all $\binom{k}{d}$ uniform hyperedges inside the planted solution have a bias $\beta$. The weights of all hyperedges are then perturbed with Gaussian noise, and the problem consists in detecting the necessary signal-to-noise ratio (snr) such that it is possible to recover the planted set given the noisy observations. 

In this paper, we study \emph{information-theoretic} (IT) and \emph{computational} algorithmic limits for this class of recovery problems, given by the parameters $d$, $k$ and $p$.
In particular, we are interested in the existence of the so-called \emph{statistical-to-computational} (SC) gaps in the \emph{sparse} regime $k \in o(p)$.

\subsection{Contribution}
We here summarize the two main contributions given in this paper.
\par (i) We provide upper and lower bounds on the recoverability threshold according to the \emph{vectorial} maximum-likelihood estimator (that is generally non-tractable). These thresholds are located at finite values according to a rescaled snr $\gamma$ defined in \Cref{eq:rescaling}, and are hence almost tight compared to previous bounds in \cite{corinzia2019}. The upper bound is based on a union bound analysis, see \Cref{th:UB}. The lower bound is a major technical contribution as it requires to address the dependencies between solution weights. The analysis is performed with two main arguments according to different regimes: recent tail bounds for the maximum of Gaussian r.v. with bounded correlation in \Cref{th:LB:max-Gaussians} and the generalized Fano's inequality in \Cref{th:infoTheoretic}, both combined with a classic covering argument from information theory. For $d \to +\infty$, the lower bound matches the IT threshold provided recently in \cite{niles2020all} for the tensorial MMSE estimator.
\par (ii) We provide a heuristic derivation of approximate-message-passing (AMP) algorithm for our problem, together with a state evolution analysis. The derivation generalizes the tensor-PCA AMP algorithm with a non-factorizable prior distribution. The recovery threshold for this class of algorithms is reported in \Cref{claim:amp_threshold}. 

\subsection{Related Work}
The study of the recovery of a planted signal in a probabilistic generative model has received much attention recently, as it constitutes a fertile ground for the analysis of the SG gaps.  
Many variations of the planted problem have been addressed in the literature. 
The stochastic block model for community detection in graphs \cite{abbe2016exact,mossel2015consistency,chen2014statistical} or hypergraphs \cite{barak2016nearly,ghoshdastidar2014consistency}) has been one of the first model to be studies and does not present any SC gap \cite{abbe2016exact,kim2018stochastic}. Despite these models have been mainly used with discrete Bernulli random variables, recent extension to \emph{weighted} edges have been proposed (see \cite{aicher2014learning,peixoto2018nonparametric}). 
Analogously, the dense matrix-PCA problem \cite{dia2016mutual, deshpande2014information} 
has been shown not to have any SC gap, as the proposed AMP algorithms match the statistical thresholds \cite{deshpande2014information}.

SC gaps have been first observed in the context of dense tensor-PCA \cite{richard2014statistical,hopkins2015tensor,montanari2016limitation,jagannath2020statistical,arous2020algorithmic, perry2020} and recently in the \emph{sparse} matrix \cite{barbier2020all} and tensor extension \cite{niles2020all, corinzia2021maximumlikelihood}. The last two works address only the statistical phase transition and are described in detail as the closest to our work. 
In \cite{niles2020all}, the statistical threshold for the MMSE estimator is fully characterized for the sparse tensor-PCA problem, and it is shown that the estimator undergoes an all-or-nothing transition. After proper rescaling, the snr threshold given in \cite{niles2020all} for the MMSE estimator is located exactly at the lower bound threshold here proved in \Cref{th:LB:max-Gaussians}. The characterization is, however, performed for the tensorial (multi-dimensional) MMSE estimator. This estimator is allowed to return any tensor in the unit ball, with no guarantee whether this can be used to recover the \emph{vectorial} planted signal. Moreover, the precise relation between the MMSE and the MLE here addressed is generally still unknown, as mentioned in \cite{niles2020all}. Indeed, \cite{corinzia2021maximumlikelihood} showed that a stronger condition on the MMSE behaviour produces an equivalent transition on the tensorial MLE, which imposes limitations on the sparsity regime for which the results apply. 

Other results on tensor-PCA bounds \cite{richard2014statistical,hopkins2015tensor,montanari2016limitation} with generic signal prior are not tight in most of the cases if applied to our problem in a black-box fashion (see, e.g. the comparison with prior upper and lower bounds given in \ifdefined\LONG \Cref{sec:TPCA}\else the appendix of the long paper\fi). This is due to the specific combinatorial structure of the planted vector (that is, in turn, a restriction on the signal prior) that is not specified for the generic spherical prior tensor-PCA.

Regarding the computational thresholds, the algorithm here described is an extension of the tensor-PCA AMP algorithm to sparse hypergraph settings. These algorithms have been shown to be information-theoretic optimal in numerous high dimensional statistical estimation problems (dense matrix-PCA \cite{deshpande2014information}, SBM \cite{abbe2018proof} etc.) and to outperform other class of algorithms in the case of SC gaps (see e.g., planted clique problem \cite{deshpande2015finding} and sparse matrix-PCA \cite{barbier2020all}). Nonetheless, AMP algorithms have been shown to underperform in the tensor-PCA problem \cite{lesieur2017statistical} to the sum-of-squares class \cite{hopkins2015tensor} and recently to averaged gradient descent \cite{biroli2020iron}. However, a recent work suggests that a hierarchy of such AMP algorithms may actually match the performance of the best known efficient algorithms \cite{Kikuchi-hierarchy}.
\section{Setting}
We study the problem of recovery a \emph{planted} sub-hypergraph on a $d$-uniform hypergraph over $p$ nodes. 
Every subset of $d$ nodes $\{i_1,i_2,\ldots,i_d\}$ is an \emph{hyperedge} whose weight $Y_{i_1,i_2,\ldots,i_d}$ is a Gaussian random variable, defined according to the following process. We denote by $\bx \in \mathcal{C}_{p,k} \subset \{0, 1\}^p$ the vector of selected nodes, with 
$$\mathcal{C}_{p,k} \coloneqq \{\bx \in \{0,1\}^p \colon \sum_i x_i = k \}.$$
Furthermore, we assume that $\bx$ is drawn uniformly at random in $\mathcal{C}_{p,k}$ with probability $\P_p$. 
The resulting weights are given by the $d$-order tensor $\bY \coloneqq \bY(\bx)$ in which all edges indicated by $\bx$ have a \emph{bias} $\beta\geq 0$, and the weights are perturbed by adding Gaussian \emph{noise} across all hyperedges $(i_1 i_2 \cdots i_d)$ with $i_1 < i_2 < \dots < i_d$.
For convenience let us now introduce the following tensor notation. The outer product of two tensors $\bU \in \bigotimes^{d_1} \reals^p$ and $\bV \in \bigotimes^{d_2} \reals^p$ is denoted by $\bU\otimes \bm{V}$ with entries $(\bU \otimes \bm{V})_{i_1 i_2\ldots i_{d_1} j_1j_2\cdots j_{d_2}} = \bU_{i_1i_2\cdots i_{d_1}}\bV_{j_1j_2\ldots j_{d_2}}$. For $\bm{x} \in \reals^p$, we define $\bm{x}^{\otimes d} = \bm{x} \otimes \cdots \otimes \bm{x} \in \bigotimes^d\reals^p$ as the $d$-th outer power of $\bm{x}$. The inner product of the two tensors $\bU \in \bigotimes^{d_1} \reals^p$ and $\bV \in \bigotimes^{d_2} \reals^p$ with $d_2 \le d_1$ is defined as 
$\< \bU , \bV \> = \sum_{j_1, \ldots ,j_{d_2}} \bU_{i_1, \ldots, i_{d_2-d_1}, j_1, \ldots
,j_{d_2}} \bV_{j_1, \ldots ,j_{d_2}} .$ 
Given a $d$-th order tensor $\bU \in \bigotimes^d \reals^p$, we define the map $\bU :\reals^p\to
\reals^p$ as
\begin{align}
\label{eq:mapping-tensors}
\bm{U}\{ \bm{x} \}_i = \sum_{i_2, \ldots , i_{d}}
U_{i,i_2, \ldots , i_{d}} \ x_{i_2} \cdot \ldots \cdot x_{i_{d}}.
\end{align} 
Using tensor notation, the observation model reads:
\begin{align}
\label{eq:weight-def}
\bY = (\beta \bx^{\otimes d} + \bZ) \mathbbm{1}_{\{i_1 < \dots < i_d\}}
\end{align}
where $\mathbbm{1}_{\{P\}} = 1$ if $P$ is true, and $0$ otherwise, and $Z_{i_1 i_2 \cdots i_h} \sim \mathcal{N}(0, 1)$. Note the main difference with the tensor-PCA formulation of the problem, where all the elements of the tensor are observed. For any signal $\bx$, we also consider the sum of all $\binom{k}{d}$ weights of the hyperedges with nodes in $\bx$ as: 
\begin{align}
S(\bx) &\coloneqq \sum_{i_1 < i_2 < \ldots < i_d} Y_{i_1, i_2 \ldots,i_d} x_{i_1} \cdot x_{i_2} \dots \cdot x_{i_d} \nonumber \\
& = \sum_i \bY\{\bx\}_i = \langle \bY, \bx^{\otimes d}\rangle \label{eq:sum-all-weights}
\end{align}

\begin{definition}[Partial and exact receovery]
\label{def:recovery}
A $k'$-partial recovery is achieved if there exists an estimator $\hat \bx$ that, with input the weight tensor $\bY(\bx)$ given by \eqref{eq:weight-def}, returns $\hat \bx = \hat{\bx}(\bY)$ such that 
$$\P(\langle \hat{\bx}, \bx \rangle \ge k') = 1 - o(1).$$
Exact recovery is achieved if
$$\P(\langle \hat \bx , \bx \rangle = k) = \P(\hat \bx = \bx) = 1 - o(1).$$
\end{definition}

\begin{definition}[Maximum-likelihood estimator]
\label{def:MLE}
The vectorial maximum-likelihood estimator is defined as 
$$\bx_{\MLE}(\bY) = \argmax\limits_{\hat \bx \colon \bY \to \hat{\bx}(\bY) \in \mathcal{C}_{p,k}} \P(\bY|\hat{\bx(\bY)})$$
We define as $P_r^{(k')} = \P(\langle \bx_{\MLE}, \bx \rangle \ge k')$ and $P_r$ accordingly.
\end{definition}

It is easy to see that the vectorial MLE estimator for the problem in \Cref{eq:weight-def} corresponds to the $k$-densest sub-hypergraph (from this the name of the problem, see Theorem~4 in \cite{corinzia2019} for a proof)
\begin{align}
\label{eq:mle_estimator}
\bx_{\MLE}(\bY) &= \argmax\limits_{\hat \bx \colon \bY \to \hat{\bx}(\bY) \in \mathcal{C}_{p,k}} \sum_{i_1 < \dots < i_d} Y_{i_1, \dots, i_d} \hat x_{i_1} \cdot \dots \cdot \hat x_{i_d} \nonumber \\
&\hspace{-1cm} =  \argmax\limits_{\hat \bx \colon \bY \to \hat{\bx}(\bY) \in \mathcal{C}_{p,k}} \langle \bY, \hat{\bx}^{\otimes d} \rangle = \argmax\limits_{\hat \bx \colon \bY \to \hat{\bx}(\bY) \in \mathcal{C}_{p,k}} S(\hat{\bx})
\end{align}

Our bounds depend on the scale-normalized snr:  \begin{equation}
\label{eq:rescaling}
\gamma \coloneqq  \beta \sqrt{\frac{\binom{k}{d}}{k}\cdot \frac{1}{2\log p}}.
\end{equation}
This scaling incorporates the parameters $k$ and $d$ of the problem and it will result in information-theoretic thresholds located at finite values. 

\begin{remark}
\label{rem:SNR}
The  scale-normalized snr $\gamma$ in \Cref{eq:rescaling} can be seen as the effective snr of the problem, given by \emph{total signal} / \emph{total noise}. The total signal is $\beta$, times the number of planted edges, hence $\beta \binom{k}{d}$. The total noise is the standard deviation $\sqrt{\binom{k}{d}}$ times the scale of the number of solutions $\sqrt{2\log \binom{p}{k}} \approx \sqrt{2k \log p}$. The latter rescale has the following intuitive justification. If we assume that $\binom{p-k}{k} \approx \binom{p}{k}$ \emph{unbiased} solutions are \emph{independent}, then their \emph{maximum} is located at $\sqrt{\binom{k}{d}}  \sqrt{2 \log p}$. The total signal has then to exceed this quantity, in order for the recovery to be possible. This argument ignores the dependencies between solutions, but it provides the right scaling of the snr. 
\end{remark}
Note that all these parameters may depend on $p$, that is, we  consider $k=k_p$ and $d = d_p$. Throughout the paper, we hide the dependency on $p$ for readability. 
In most of the analysis, the following rescaling of the involved quantities will appear naturally
\begin{equation}
\label{eq:rate}
\rate{q} \coloneqq \lim_{p \rightarrow +\infty} \frac{\log q}{\log p}
\end{equation}
which intuitively means that $q \approx  n^{\rate{q}}$. We call $\rate{q}$ the \emph{rate} of a generic $q=q_p$.
\section{Information-Theoretic Bounds}
\label{sec:infoTheoryandUpperBounds}
By the MLE estimator's characterisation given in \Cref{eq:mle_estimator}, the recovery regime is regulated by the weight $S(\bx)$ of the planted solution and how it compares to the best among all other solutions' weights. For the analysis, it is useful to partition the latter according to their overlap with the planted solution. 

\begin{lemma}
\label{le:fail-rec-union}
For any $m \in \{0, \ldots, k\}$, let $$\mathcal{S}_m = \{ \hat \bx \in \mathcal{C}_{p,k} \colon \langle \bx, \hat \bx \rangle = m\}$$ denote the set of all solutions that share exactly $m$ nodes with the planted solution $\bx$. Then the following bound holds for all $k' \in\{0,\ldots,k\}$:
\begin{align}
\label{eq:UB:max-based}  
1 - P_r^{(k')} &\leq \sum_{m=0}^{k'-1} \P\left( S(\bx) < \max_{\hat \bx \in \mathcal{S}_{m}} S(\hat \bx)\right).
\end{align}
For all $m <k$ it further holds:
\begin{align}
\label{eq:LB:max-based}
1 - P_r &\geq  \P\left( S(\bx) < \max_{\hat{\bx} \in \mathcal{S}_m} S(\hat{\bx})\right).
\end{align}
\end{lemma}

The proof of this lemma is given by a union bound and is given in the \ifdefined\LONG appendix\else long version of the paper\fi. Given \Cref{def:MLE} and the latter inequalities, we can reduce the analysis of the recovery regime to the determination of the values of $\gamma$ for which $\P\left(S(\bx) < \max_{\hat{\bx} \in \mathcal{S}_m} S(\hat{\bx})\right)$ vanishes or has limit 1, for $p \to +\infty$. 
In the first scenario (recovery regime), these probabilities need to vanish \emph{sufficiently fast} to apply the union bound in \Cref{eq:UB:max-based} over all different $m$.

\subsection{Upper bound (partial or exact recovery)}
\label{sec:UB}
We here provide upper bounds on the failure probability of partial and exact recover. The proof is given in \ifdefined\LONG appendix\else long version of the paper\fi.

\begin{theorem}
\label{th:UB}
For any $k$ and any $k'\in\{1,\ldots, k\}$, and for any $\gamma > \gamma^{({k'})}_{\UB}$, the MLE estimator achieves $k'$-partial recovery according to \Cref{def:recovery}. 
The critical gamma is defined as 
$$\gamma^{({k'})}_{\UB}  \coloneqq \sqrt{1+ \rate{k}-2 \rate{k - k'}+\rate{k'}} + \sqrt{\rate{k} - \rate{k - k'}+\rate{k'}}$$ where $\rate{k}$, $\rate{k'}$ and $\rate{k-k'}$ are defined according to \Cref{eq:rate}.
\end{theorem}
It follows easily from the latter theorem, the following on exact recovery.
\begin{lemma}
\label{le:exact_recovery}
Exact recovery is achieved for $\gamma > \gamma_{\UB}$, with 
$$\gamma_{\UB} = \sqrt{1+2\alpha_k} + \sqrt{2\alpha_k}.$$
Recovery of a constant fraction of $k$ nodes is achieved with $k' = \lambda k$, with $\lambda$ constant, for $\gamma > \gamma_{\UB}^{(\lambda k)}$ with 
$$\gamma_{\UB}^{(\lambda k)} = 1 + \sqrt{\alpha_k}.$$
\end{lemma}

\subsection{Lower bounds (impossibility of recovery)}
We here provide a characterisation of the regime where recovery is impossible, given by the following two theorems, valid in different regimes.

\begin{theorem}
\label{th:LB:max-Gaussians}
For $\rate{k} \in (0,1)$, and for $d \in \omega(1)$ the following holds. $$\lim_{p\to +\infty} P_r = 0$$ for any $\gamma < \gamma_{\LB,\G}$, with $\gamma_{\LB,\G}$ given by:
\begin{align}
\label{eq:lambda-LB-cor}
\gamma_{\LB,\G} \coloneqq 
\begin{cases}
\sqrt{1 - \alpha_k} & \text{ for } d \in o(\sqrt{k})
\\
\sqrt{1 - \alpha_k}/\sqrt{e} & \text{ otherwise }
\end{cases} 
\end{align}
\end{theorem}

\begin{theorem}
\label{th:infoTheoretic}
For $\rate{k} \in (0,1)$ and any $d$, the following holds. $$\limsup\limits_{p\to +\infty} P_r < 1$$ for any $\gamma < \gamma_{\LB,\F}$, with $\gamma_{\LB,\F}$ given by:
\begin{align}
\gamma_{\LB,\F} \coloneqq \sqrt{\frac{1 - \rate{k}}{2}}
\end{align}
\end{theorem}

The first theorem gives tighter bounds, but its validity is confined to the case where the order of the hypergraph (or tensor) $d$ grows to infinity with $p$. The second theorem is valid in any regime. However, it provides only an impossibility result as $\limsup P_r < 1$, and achieves a lower threshold.
The proofs in the two regimes use two main arguments that are, respectively: (i) A recent tail bound on the maximum of \emph{dependent} Gaussians with bounded correlation (see \cite{lopes2018maximum} \ifdefined\LONG and \Cref{le:max-Gaussians}\else\fi) and (ii) the generalised Fano's inequality. Both proofs consider a coverage set of \emph{weakly overlapping} solutions defined below. Intuitively, recovery in the original problem is at least as difficult as the recovery restricted to this set of weakly dependent solutions if the coverage is \emph{sufficiently large}.

\begin{definition}
\label{def:cover}
For any $r \in \{0,\ldots,k\}$ define the coverage with overlap $r$ a subset $\mathcal{C}(r) \subset \mathcal{C}_{p,k}$ of solutions satisfying the following conditions:
(i) Any two solutions in $\mathcal{C}(r)$ share less than $r$ nodes.
(ii) For any solution $\bx' \in \mathcal{C}_{p,k}, \bx' \not \in \mathcal{C}(r)$ there exists a solution $\bx'' \in \mathcal{C}(r)$ such that $\bx'$ and $\bx''$ have at least $r$ nodes in common.
We denote by $C(r) \coloneqq |\mathcal{C}(r)|$ the cardinality of the coverage.
\end{definition}
	
In the proof of \Cref{th:LB:max-Gaussians}, we use the above mentioned result \cite{lopes2018maximum} to show that the maximum among solutions in $\mathcal{C}(r)$ concentrates tightly around $\sigma_k  \sqrt{\rate{C(r)}\cdot 2\log p}$ where $\sigma_k = \sqrt{\binom{k}{d}}$, while the planted solution concentrates around its expectation $\beta_k = \beta \binom{k}{d}$. Then, the condition for recovery translates into $\beta_k > \sigma_k  \sqrt{\rate{C(r)}\cdot 2\log p}$ which corresponds to $\gamma < \sqrt{\frac{\rate{C(r)}}{k}}$. A crucial point is that, in order to prove the concentration of the maximum over $\mathcal{C}(r)$, the overlap $r$ has to be small enough such that the maximum correlation between solutions weights $S(\hat{\bx})$ vanishes, which gives the conditions on $d$ in \Cref{th:LB:max-Gaussians} and the corresponding bound on $\rate{C(r)}$.

The proof of \Cref{th:infoTheoretic} is based on Fano's inequality. By restricting to the solutions in $\mathcal{C}(r)$, we use the generalized Fano's inequality \cite{verdu1994generalizing} to show that $P_r$ satisfies the bound $P_r \lesssim 1- \I(\bx;\bY) \cdot \log^{-1} C(r)$
where $\I(\bx;\bY)$ is  the mutual information between the planted solution $\bx$ and the observations $\bY(\bx)$. By using \cite{verdu1994generalizing} and the combinatorial structure of $\mathcal{C}_{p,k}$, we can upper bound the mutual information as $\binom{k}{d}\beta^2$. The rescaling of $\gamma$ and a lower bound on the cardinality of the coverage set $C(r)$ gives the claim. The proof of both theorems is given in \ifdefined\LONG the appendix\else the long version of the paper\fi.
\section{Computational Thresholds via Approximate Message Passing}
\label{sec:amp}
\begin{figure*}[ht]
\centering
\includegraphics[width=1.\textwidth]{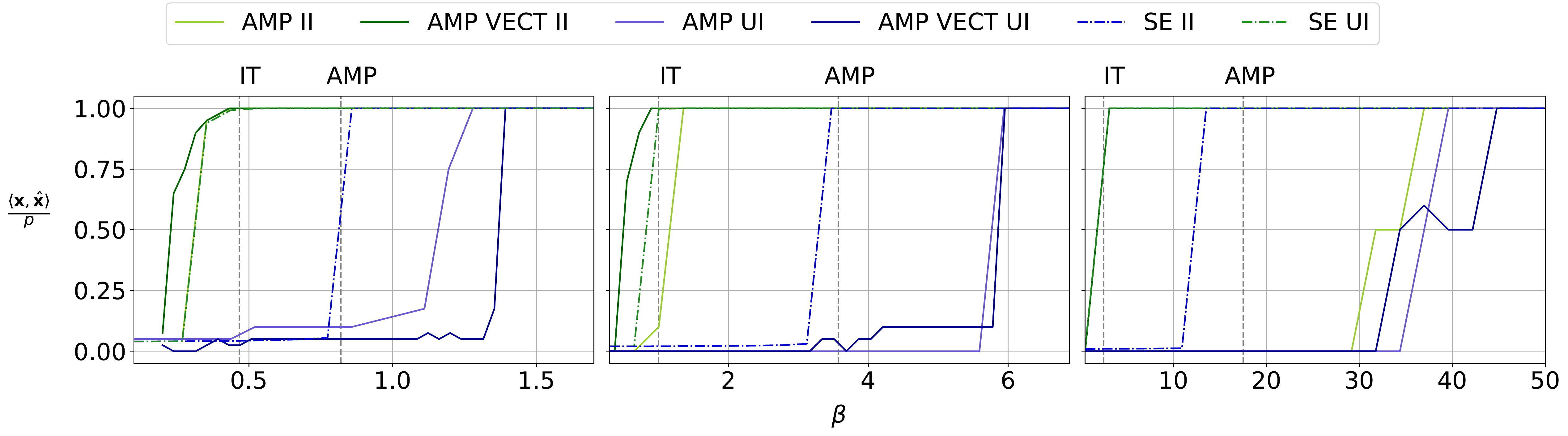}
\caption{Empirical performance of the AMP algorithm for fixed $p=500$, $d=3$ and different values of $k$ (respectively, from left to right, $k=20, 10, 5$). The experiment is repeated 20 times, and we report the median overlap achieved. The AMP and AMP VECT refer respectively to using a scalar and a vectorial thresholding function as in \Cref{eq:f_multidimensional}.}
\label{fig:amp_perf}
\end{figure*}
\begin{figure}[ht]
\centering
\includegraphics[width=0.48\textwidth]{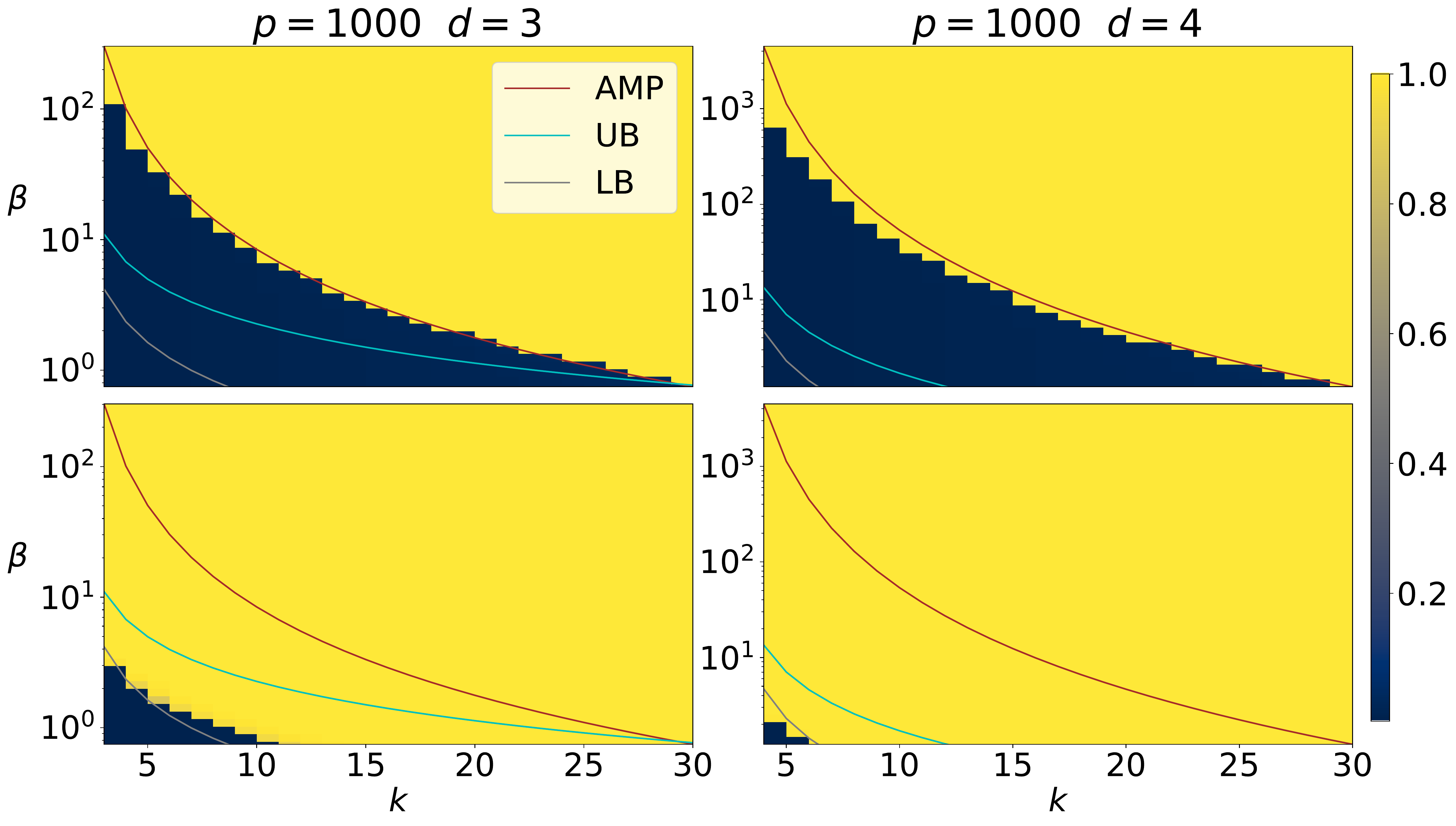}
\caption{State evolution fixed point overlap $m^*$ reached by the factorized equation in \Cref{eq:se_factorized}. The first row reports the overlap obtain with uninformative initialization (UI, $m_0 \approx 0$), while the bottom row reports the overlap with informative initialization $m_0 = 1$. The empirical fixed point with UI well matches the analytical AMP threshold given by \Cref{claim:amp_threshold} (brown line, expressed in terms of the $\beta$ parameter), while the II fixed point approaches the conjectured IT threshold located at the lower bound.}
\label{fig:state_evolution}
\end{figure}
Approximate message passing algorithms are a class of algorithms for high dimensional statistical estimation that approximate belief propagation in the large system limit \cite{donoho2009message}. Intuitively, it iteratively estimates the mean $\bm{x}^{(t)}$ and the variance $\bm{a}^{(t)}$ of the classic belief propagation messages in the factor graph of the estimation problem, discarding low order terms that depend on the target factor node of the messages. The algorithm results in a slight modification of a general spectral algorithm for tensor-PCA. In the following, we use the operator `$\circ$' to denote an operation that is performed elementwise on a vector or tensor.
Let introduce the following probability distribution 
\begin{equation*}
g(\bm{y} | \bm{a}, \bm{x}) = \frac{\P_p(\bm{y}) \exp(\bm{x} \cdot \bm{y} - (\bm{y} \circ \bm{y}) \cdot \bm{a}) / 2 }{Z(\bm{a}, \bm{x})}
\end{equation*} 
where $Z(\bm{a}, \bm{x})$ is the normalizing constant. Then, we define the following threshold function which will be used in the definition of AMP algorithm below: 
\begin{equation}
\label{eq:threshold_f_multidimensional}
f(\bm{a}, \bm{x}) = \mathbb{E}_{g(\cdot|\bm{a}, \bm{x})} [y].
\end{equation}

\subsection{The AMP procedure}
The iterative AMP procedure reads (derivation given in \ifdefined\LONG \Cref{app:amp}\else long version of the paper\fi):
\begin{align}
&\bm{b}^{(t)} = \beta^2 (d-1) \langle \bm{Y}^{\circ 2},  \bm{\sigma}^{(t)} \otimes (\bm{\hat{x}}^{(t)} \circ \bm{\hat{x}}^{(t-1)})^{\otimes d-2} \rangle \nonumber\\
&\bm{x}^{(t)} = \beta \mathbf{Y}\{\bm{\hat{x}}^{(t)}\} - \langle \bm{b}^{(t)}, \bm{\hat{x}}^{(t-1)}\rangle \nonumber\\
&\bm{a}^{(t)} = \beta^2 \bm{Y}^{\circ 2} \{\bm{(\hat{x}}^{(t)})^{\circ 2}\}
\label{eq:amp_algorithm} \\
&\bm{\sigma}^{(t+1)} = \bm{diag}(\bm{J}_x f(\bm{a}^{(t)}, \bm{\hat{x}}^{(t)})) \nonumber\\
&\bm{\hat{x}}^{(t+1)} = f(\bm{a}^{(t)}, \bm{x}^{(t)} ) \nonumber
\end{align}
where $\bm{J}_x f$ is the Jacobian w.r.t varible $x$ of the function $f$ and  $\bm{diag}(\cdot)$ extracts its diagonal entries. 
The threshold function $f$ becomes tractable when the prior distribution $\P_p$ \emph{factorizes} as $\P_p(\bm{x})=\prod_i p(x_i)$. In this case, the threshold function $f$ factorizes as well into independent components, and the AMP equations proposed here are equivalent to those in \cite{lesieur2017statistical} (more details in \ifdefined\LONG \Cref{app:amp}\else long version of the paper\fi). 
In our problem, however, the prior distribution does \emph{not} factorize since $\P_p$ is the uniform distribution over $\mathcal{C}_{p,k}$. A similar issue arises in the planted clique problem studied in \cite{deshpande2015finding}, where the authors propose to approximate the prior distribution with a factorized Bernulli distribution with parameter $\delta = k/p$. With this approximation, the threshold function reads $f(\bm{a}, \bm{x}) = (f_i(a_i, x_i) )_{i=1}^p$, where 
\begin{align*}
f_i(a_i, x_i) = \frac{1}{1 + \exp(-x_i + a_i/2 + \log(1/\delta - 1)}\ .
\end{align*}
While this Bernulli approximation is effective for large $k$ -- like the regime $k = \Theta(\sqrt p)$ studied in \cite{deshpande2015finding} -- it may be inaccurate for small $k$. We indeed observe experimentally (see \Cref{sec:experiments}) that for small values of $k$, the above approximation is no longer effective. Hence, we propose a finer approximation of the threshold function that is still tractable. Our parametrization is inspired by the independent Bernoulli approximation:
\begin{equation}
\label{eq:f_multidimensional}
f(\bm{a}, \bm{x}) = \left(\frac{1}{1 + \exp(-x_i + a_i/2 + \lambda(\bm{x})} \right)_{i=1}^p
\end{equation}
with $\lambda(\bm{x})$ being a scalar used to enforce $f$ to select $k$ components equals to $1$, hence given by $\sum_i f_i(\bm{a}, \bm{x}) = k$. The experiments described in \cref{sec:experiments} below show that our finer approximation outperforms the Bernulli i.i.d. approximation.  

\subsection{State Evolution}
The evolution of the approximated message passing algorithm in the special case of Bayesian-optimal inference, can be tracked by a one dimensional iterative equation of the overlap order parameter. In our setting, with a generic (non-factorizable) prior distribution $\P_p$ and threshold function $f$ we can define the multidimensional overlap order parameter as
$$\bm{m}^{(t)} = \frac{1}{\binom{p-1}{d-1}} \mathbb{E}_{\bm{x}}[\bm{1}\{\bm{x} \circ \bm{x}^{(t)}\}] $$
where $\bm{1}$ is the tensor with components $\bm{1}_{i,i_2,\dots,i_d} = 1$ if $i_2 < i_3 < \dots < i_d$ and $0$ otherwise. The multidimensional state evolution (SE) (generalizing \cite{lesieur2017statistical}) reads then (see \ifdefined\LONG \Cref{app:amp} \else the long version of the paper\fi for a heuristic derivation):
\begin{align*}
&\bm{m}^{(t+1)} = \frac{\E_{\bm{x},\bm{z}}\left[ \bm{1}\{\bm{x} \circ f\left(\bm{\hat{m}}^{(t)}, \bm{\hat{m}}^{(t)} \circ \bm{x} + (\bm{\hat{m}^{(t)}})^{\circ 1/2} \circ \bm{z} \right) \} \right]}{\binom{p-1}{d-1}}
\end{align*}
where $\bm{\hat{m}}^{(t)} = \beta^2 \binom{p-1}{d-1} \bm{m}^{(t)}$ and $\bm{z}$ is a $p$-dimensional vector with i.i.d. standard Gaussian entries. Note again that assuming a factorized i.i.d. prior distribution $\P_p(\bm{x}) = \prod_i p(x_i)$ the SE is equivalent to the single letter evolution of the \emph{scalar} overlap $m_t = \frac{1}{p} \langle \bm{x} , \bm{\hat x}^{(t)} \rangle$ described in \cite{lesieur2017statistical},
\begin{equation}
\label{eq:se_factorized}
m_{t+1} = \mathbb{E}_{x, z}\left[x \cdot f(\hat{m}_t, \hat{m}_t x + \sqrt{\hat{m}_t} z) \right],
\end{equation}
where $x \sim p$, $z \sim \Gaussian(0,1)$, and $\hat{m}_t = \beta^2 \binom{p-1}{d-1} m_t$. In the following we derive an analytical threshold for the AMP algorithm to succeed, approximating the true prior distribution with a factorized Bernulli prior with parameter $\delta=\frac{k}{p}$ and using the simplified SE in \Cref{eq:se_factorized}. A simple heuristic argument based on the study of the fixed points of the SE is given in\ifdefined\LONG \Cref{app:amp}\else the long version of the paper\fi.
\begin{claim}
\label{claim:amp_threshold}
The recovery threshold for the AMP algorithm reads:
\begin{equation}
\gamma_{\AMP} \coloneqq \sqrt{\frac{1}{2 e} \left(\frac{p}{k}\right)^{d-1} \frac{1}{d(d-1)\log p}}.
\end{equation}
\end{claim}
\section{Experiments and Discussion}
\label{sec:experiments}
In \Cref{fig:amp_perf}, we report both the empirical performance of the AMP algorithm and the factorized SE fixed point (according to \Cref{eq:se_factorized}) for uninformative initialization (with a random overlap with the planted solution) and informative initialization (respectively UI and II) for different values of $k$. 
For a high value of $k$, the Bernoulli i.i.d. approximation of the prior function and factorized thresholding function well matches the SE fixed point, with the UI empirical performance slightly worse than the SE prediction due to finite-size effects (as already observed for the planted clique problem in graphs \cite{deshpande2015finding}). At smaller values of $k$, the statistical dependency between signal components increases and the AMP empirical performance heavily mismatch the SE prediction. In this setting, the proposed multivariate threshold function significantly outperforms the naive factorized AMP, matching the SE in the II setting correctly. Interestingly, the dynamic phase transition of the AMP indicated by the critical signal at which the SE with II fails approaches the conjectured information-theoretic threshold $\gamma_{\LB}$. This finding suggests a possibility of analysis of the IT thresholds with statistical physics-inspired techniques (e.g., interpolation methods \cite{barbier2020all}) that so far have never been applied rigorously to the case of structured priors in the sparse setting. In \Cref{fig:state_evolution} we also report the factorized SE fixed point for  $d=3$ and $d=4$ for both UI and II. We can observe the good agreement between the analytical computational threshold given by \Cref{claim:amp_threshold} and the SE's empirical fixed point. The dynamical phase transition indicated by the $m^*$ transition for II (bottom row) is close to the information-theoretic lower bound (as shown already for the tensor-PCA problem \cite{lesieur2017statistical}).

\ifdefined\LONG 
\section*{Acknowledgement}
This work was supported in part by
NSF Center on Science of Information
Grants CCF-0939370 and NSF Grants CCF-1524312, CCF-2006440, CCF-2007238.
\else \fi

\begin{appendices}
\section{Postponed proofs}
\begin{proof}[Proof of \Cref{le:fail-rec-union}]
By definition of $P_r^{(k')}$, $1-P_r^{(k')} = \P\left(\langle \bx_{\MLE} , \bx \rangle < k'  \right)$. By the characterization of the MLE in \Cref{eq:mle_estimator} and the definition of $\mathcal{S}_m$, the latter event is equal to 
$$\{S(\bx) > S(\hat{\bx}), \ \forall \ \hat{\bx} \in \mathcal{S}_0\cup \cdots \cup \mathcal S_{k'-1}\}.$$ The claim in \Cref{eq:UB:max-based} follows then from the union-bound on the probability
$$\P\left(\bigcup_{m=0}^{k'-1} \left\{S(\bx) \le \max_{\hat{\bx} \in \mathcal{S}_m}(S(\hat{\bx}))\right\}\right).$$
To prove the lower bound in \Cref{eq:LB:max-based}, we can observe that if $S(\bx) < \max_{\hat{\bx} \in \mathcal{S}_m}(S(\hat{\bx}))$ for some $m < k$, then the $\bx_{\MLE} \neq \bx$, and thus it fails to exactly recover the planted solution.
\end{proof}

\subsection{Proofs for the Lower Bound}
\begin{proof}[Proof of \Cref{th:UB}]
For the analysis, we define the following quantities depending on $k' \in \{0,\ldots, k-1\}$ (we consider $p, k, d$ to be the parameters of the problem, hence their dependency is not highlighted):
\begin{align*}
&Q(k') \coloneqq \binom{p-k}{k-{k'}} \\
&M(k') \coloneqq \binom{k}{k'}\binom{p-k}{k-{k'}} = \binom{k}{k-{k'}}\binom{p-k}{k-{k'}} \\
&D(k') \coloneqq \binom{k}{d} - \binom{k'}{d}.
\end{align*} 
For each fixed subset of $k'$ nodes of the planted solution, there are $Q(k')$ solutions that share exactly these $k'$ nodes with the planted solution. 
Moreover, there are exactly $M(k')$ solutions that share any $k'$ nodes with the planted solution. 
Each solution sharing $k'$ nodes with the planted solution differs in $D(k')$ edges with the latter.
We use the union bound given in \Cref{eq:UB:max-based}, and control the quantities
$$\P\left( \max_{\hat{\bx} \in \mathcal{S}_m} S(\hat{\bx}) > S(\bx) \right)$$ 
using the tail bounds of the Gaussian from \Cref{le:gaussian_tail} and spitting the inequality into the sum of two independent terms as in \Cref{le:inequility_t}. We hence get the following Lemma (full proof given below in this section): 
\begin{lemma}
\label{le:UB_epsilon}
For every $\epsilon > 0$ and for every $k' < k$ and $k > 1$, let $\gamma^{(k')}_{UB_{\epsilon}} \coloneqq \sqrt{\frac{\binom{k}{d}}{k D(k')}}\cdot UB_{\epsilon}(k')$, where 
\begin{align}
\label{eq:UB:gamma_k}
UB_{\epsilon}(k') \coloneqq \sqrt{\frac{\log M(k')}{\log p}+\epsilon}  + \sqrt{\frac{\log \binom{k}{k'}}{\log p}+\epsilon}  \ .	  
\end{align}
Then, for any $\gamma > \gamma^{({k'})}_{UB_{\epsilon}}$ it holds that
$$\P\left(\max_{\hat \bx \in \mathcal{S}_{k'}} S(\hat \bx) >  S(\bx) \right) \le \frac{1}{\sqrt{\pi} p^\epsilon}.
$$
\end{lemma}
Plugging the result of the latter Lemma into \Cref{eq:UB:max-based} we get:
$$1-P_r^{(k')} \in  \mathcal{O}\left(\frac{k'}{p^{\epsilon}} \right)\ .$$
Using now \Cref{le:gamma-constants} to characterize further the bound $\gamma^{({k'})}_{UB_{\epsilon}}$, and reparametrazing $\epsilon = \rate{k'}+\tilde{\epsilon}$, with $\tilde{\epsilon} > 0$ an arbitrarily constant, we have 
$$\frac{k'}{p^\epsilon} = \frac{k'}{p^{\rate{k'}}} \cdot \frac{1}{p^{\tilde{\epsilon}}} = \exp\left(\log p \left[- \tilde{\epsilon} + \frac{\log k'}{\log p} - \rate{k'}\right]\right) \in o(1).$$
The asymptotics is due to the fact that $\frac{\log k'}{\log p} - \rate{k'} \rightarrow 0$, by definition of $\rate{k'}$ in \Cref{eq:rate}, and hence the expression $- \tilde{\epsilon} + \frac{\log k'}{\log p} - \rate{k'}$ is negative for sufficiently large $p$. 
Hence, the probability $1-P_r^{(k')}$ tends to $0$.
\end{proof}

\begin{proof}[Proof of \Cref{le:UB_epsilon}]
For an arbitrary subset $F$ of ${k'}$ nodes of the planted solution $\bx$, denoted by $\bx_F$, with ${k'}\in \{0,\ldots,k-1\}$, let $\mathcal S_{k'}^{F}$ be the set of all solutions that share exactly the set $F$ with the planted solution. Let $S^{-F}(\hat \bx)$ denote the sum of the weights in $\hat \bx$ but not in $\bx_F$, as 
$$S^{-F}(\hat \bx) = S(\hat{\bx}) - S(\hat{\bx} \circ \bx_F ).$$ 
By the union bound over the $\binom{k}{k'}$ possible fixed subsets $F$ of $k'$ nodes, we get 
\begin{align}
\P\left(\max_{\hat \bx \in \mathcal{S}_{k'}} S(\hat \bx) >  S(\bx)\right) &\le \nonumber\\ 
&\hspace{-1cm} \le \binom{k}{{k'}} \P\left(\max_{\hat \bx \in \mathcal{S}_{k'}^{F}} S(\hat \bx) > S(\bx)\right) \nonumber \\
\label{eq:simple_UB}
&\hspace{-1cm} = \binom{k}{{k'}} \Pr\left(\max_{\hat \bx \in \mathcal{S}_{k'}^{F}} S^{-F}(\hat \bx) > S^{-F}(\bx)\right)
\end{align}
where the equality follows from subtracting from both sides the quantity $S(\bx_F)$.
Using \Cref{le:inequility_t}, we have  
\begin{align}
\label{eq:simple_UB_two_terms}
\P\left(\max_{\hat \bx \in \mathcal{S}_{k'}^{F}} S^{-F}(\hat \bx) > S^{-F}(\bx)\right) &\le
\nonumber \\ &\hspace{-3cm} \le \P\left(\max_{\hat \bx \in \mathcal{S}_{k'}^{F}} S^{-F}(\hat \bx) > t \right) + \P\left(t \ge S^{-F}(\bx)\right) \nonumber \\ 
&\hspace{-3cm} \leq \P\left(\max_{\hat \bx \in \mathcal{S}_{k'}^{F}} S^{-F}(\hat \bx) > t_{\Delta'}\right) + \nonumber\\
& \hspace{-1cm} + \P\left(D(k')\beta - t_{\Delta''} > S^{-F}(\bx) \right) \nonumber \\
&\hspace{-3cm} \leq Q(k') p(k',\Delta') + p(k',\Delta'')
\end{align}
where the second inequality follows from  \Cref{le:existance_t} (together with the definition of $t_{\Delta'}$ and $t_{\Delta''}$) and the latter inequality follows from  \Cref{eq:UB-bound:planted} and \Cref{eq:LB-bound:m-intersecting} in \Cref{le:UB-concentration}. 
Combining \Cref{eq:simple_UB} and \Cref{eq:simple_UB_two_terms}, we get 
\begin{align*}
\P\left(\max_{\hat \bx \in \mathcal{S}_{k'}} S(\hat \bx) >  S(\bx)\right) &\le \binom{k}{k'} \left( Q(k') p(k',\Delta') + p(k',\Delta'')\right) \\
&\hspace{-3.5cm} =\binom{k}{k'} \frac{1}{\sqrt{4 \pi \log p}} \left( \left(\frac{1}{p}\right)^{\frac{\Delta'}{D(k')}} \frac{Q(k')}{\sqrt{\Delta'}} + \left(\frac{1}{p}\right)^{\frac{\Delta''}{D(k')}} \frac{1}{\sqrt{\Delta''}}\right) \\
&\hspace{-3.5cm} = \frac{1}{\sqrt{4 \pi \log p}} \Bigg(\left(\frac{1}{p}\right)^{\frac{\Delta'}{D(k')} - \frac{\log M(k')}{\log p}} \frac{1}{\sqrt{\Delta'}} + \\
&\hspace{-0.5cm} + \left(\frac{1}{p}\right)^{\frac{\Delta''}{D(k')}- \frac{\log \binom{k}{k'}}{\log p}} \frac{1}{\sqrt{\Delta''}}\Bigg)
\end{align*}
where in the first equality we used \Cref{eq:bound-generic:m-intersecting} and in the last equality the identity $z = p^{\frac{\log z}{\log p}}$ and the fact that $$\binom{k}{d} Q(k') = \binom{k}{d} \binom{p - k}{k-k'}= M(k').$$ 
Then, by \Cref{le:existance_t}, we get
\begin{align*}
\P\left(\max_{\hat \bx \in \mathcal{S}_{k'}} S(\hat \bx) >  S(\bx)\right)  &\le \frac{1}{p^\epsilon} \frac{1}{\sqrt{4 \pi \log p}}\left(\frac{1}{\sqrt{\Delta'}}+ \frac{1}{\sqrt{\Delta''}}\right) 
\\& \hspace{-2cm} \le \frac{1}{p^\epsilon} \frac{1}{\sqrt{4 \pi \log p}} \frac{1}{D(k')} \Bigg(\sqrt{\frac{\log p}{\log M(k')}} + \\
&\hspace{2cm} + \sqrt{\frac{\log p}{\log\binom{k}{k'}}}\Bigg) \\
&\hspace{-2cm} \le \frac{1}{\sqrt{\pi} p^\epsilon}
\end{align*}
where in the last inequality we used $D(k') \geq 1$, $M(k') \geq \binom{k}{k'} \ge k$, and $\log k > 1$, that valid for $k > k'$ and $k > 1$. 
\end{proof}

\subsection{Proofs of the Lower Bound}
\begin{proof}[Proof of \Cref{th:LB:max-Gaussians}]
\label{sec:proof:UB}
The proof of this lower bound is based on the result on tail bounds on the \emph{maximum} of Gaussians with \emph{bounded} correlation \cite{lopes2018maximum} given for convenience in \Cref{le:max-Gaussians}. 
We here outline a road map of the proof: 
(i) Given an overlap $r \le k$ between two solutions $\bx', \bx''$, we define the correlation $\rho(r)$ as the correlation between the random variables $S(\bx')$ and $S(\bx'')$. We first analyze the conditions for vanishing correlation $\rho(r)\rightarrow 0$. 
(ii) Using the tail bound on correlated Gaussians in \Cref{le:max-Gaussians}, we show that vanishing correlation implies the concentration of the maximum of a given coverage of solutions $\mathcal{C}(r)$.
(iii) Using a lower bound on the rate of the cardinality of the coverage, we provide a respective lower bound of the recovery threshold.
We hence first give a Lemma that characterizes the condition to have vanishing correlations.

\begin{lemma}
\label{le:vanishing-corr} 
For any two solutions $\bx', \bx'' \in \mathcal{C}_{p,k}$, that share $r$ nodes as $\langle \bx', \bx'' \rangle = r$, the correlation of the weights $\rho \coloneqq \rho_{S(\bx'), S(\bx'')}$ reads
$$\rho = \frac{\binom{\share}{d}}{\binom{k}{d}}$$ 
and it vanishes in each of the following two regimes: 
\begin{enumerate}
\item For $d \in \omega(1)$ and $\share \leq \lambda k$ for any constant $\lambda$ satisfying
\begin{align}
\label{eq:linear-ell-cases}
\begin{cases}
\lambda < 1 & \text{ for } d \in o(\sqrt{k})
\\
\lambda <1/e & \text{ otherwise }
\end{cases}
\end{align}
\item For $d \in \mathcal{O}(1)$ and for any $\share \in o(k)$. 
\end{enumerate}
\end{lemma}
We can hence show in the following that, given a coverage with vanishing correlation $\rho(r)$, the maximum of such coverage in bounded from below in high probability.
\begin{lemma}
\label{le:max-concentration}
Given a number of shared nodes $\share$ and a coverage $\mathcal{C}(r)$ that satisfies \Cref{def:cover}, if the correlation between the weights of two solutions $\bx', \bx'' \in \mathcal{C}(r)$, $\rho(\share) = \rho_{S(\bx'), S(\bx'')} \to 0$ vanishes, then for any constant $\epsilon \in (0,1)$ the following upper bound on the maximum weight in the coverage holds:
\begin{align}
\label{eq:cor-max-concentration}
\lim_{p \to +\infty} \Pr\left(\max_{\hat{\bx} \in \mathcal{C}(r)} S(\hat{\bx}) \le (1- \epsilon) \cdot \sigma_k  \sqrt{\rate{C(r)}\cdot 2\log p}\right) = 0
\end{align}
where $\sigma_k = \sqrt{\binom{k}{d}}$.
\end{lemma}

We further provide a bound on the opposite direction for the weight of the planted solution, such that the two quantities can be well separated in high probability.

\begin{lemma}
\label{le:planted-concentration}
For any sequence $\Delta = \Omega\left(\frac{\binom{k}{d}}{\log p}\right)$, it holds that 
\begin{align*}
\lim_{p \to +\infty} \Pr_z\left(S(\bx) > \beta_{k} + \sqrt{\Delta \cdot 2\log p} \right) = 0 
\end{align*}
where $\beta_k = \binom{k}{d} \beta$ and $\bx$ is the planted solution.
\end{lemma}

We can prove the main Lemma that connects the recovery threshold to the rate of any coverage $\rate{C(r)}$.

\begin{lemma}
\label{le:LB-R}
Given the assumptions of \Cref{le:max-concentration} and given that $\rate{C(r)} \in \Omega(1)$, for any $\epsilon > 0$ constant, if $\gamma \le (1-\epsilon) \sqrt{\frac{\rate{C(r)}}{k}}$, then the probability of recovery vanishes:
\begin{align*}
\lim_{p \to +\infty} P_r = 0 \ . 
\end{align*}
\end{lemma}

\begin{proof}
From the hypothesis that $\rate{C(r)} \in \Omega(1)$, $\log p \to +\infty$, and given any constant $\delta_0 > 0$ the condition on $\gamma$ above implies that for $p$ large enough:
$$\gamma \le (1-\epsilon)   \sqrt{\frac{\rate{C(r)}}{k}} - \sqrt{\frac{\delta_0}{k\log p}} \ .$$
Using the definition of $\gamma$ and defining $\Delta = \delta_0 \cdot \frac{\binom{k}{d}}{\log p}$ we get with simple manipulation:
\begin{align}
\label{eq:condition_gamma}
(1-\epsilon) \sqrt{\binom{k}{d}}  \sqrt{\rate{C(r)}\cdot 2\log p} > \binom{k}{d}\beta + \sqrt{\Delta \cdot 2\log p}
\end{align}  
Using \Cref{le:inequility_t} and \Cref{eq:LB:max-based}, we can write an upper-bound on the recovery probability as: 
\begin{align*}
P_r &\le \Pr\left(\max_{\hat{\bx} \in \mathcal{C}(r)} S(\hat{\bx}) \le \plant \right) \\
& \le \Pr\left(\max_{\hat{\bx} \in \mathcal{C}(r)} S(\hat{\bx}) \le t\right) + \Pr\left(t < S(\bx) \right) . 
\end{align*}
We can now use the condition in \Cref{eq:condition_gamma}, to get a $t$ such that the conditions for both \Cref{le:max-concentration} and \Cref{le:planted-concentration} are satisfied. We hence obtain 
\begin{align}
\label{eq:max-and-planted-vs-t}
\Pr\left(\max_{\hat{\bx} \in \mathcal{C}(r)} S(\hat{\bx}) > t\right) \rightarrow 1 && \text{ and } && \Pr\left(S(\bx) \le t \right) \rightarrow 1
\end{align}
and so the claim follows.
\end{proof}

Consider any $\share = \lambda k$ with $\lambda$ being any constant satisfying \Cref{eq:linear-ell-cases}, according to the regime of $d$ and $k$. 
By \Cref{le:vanishing-corr} and the assumption of the theorem $d \in \omega(1)$, the correlation $\rho(\share)$ vanishes. We can hence bound the recovery threshold with a respective bound on the rate $\rate{C(r)}$, that is given in the following Lemma.

\begin{lemma}
\label{le:R-lb}
There exist a coverage $\mathcal{C}(r)$ according to \Cref{def:cover} with cardinality $\mathcal{C}(r)$, $C(r) \coloneqq |\mathcal{C}(r)|$ at least
\begin{align}
\label{eq:cover-ratio}
C(r) \gtrsim \frac{\binom{p}{k}}{B(r)} && \text{for} && B(r) = \sum_{l=\share}^k \binom{k}{l}\binom{p-k}{k - l}.
\end{align}

For any $\rate{k} \in (0,1)$ and $r \leq \lambda k$, with constant $\lambda \in (0,1)$, and for $r \in \omega(\frac{k}{\log n})$ it holds that 
$\rate{C(r)} \gtrsim \share(1-\rate{k})$.
\end{lemma}

We can thus apply \Cref{le:LB-R} and \Cref{le:R-lb} with $\share = \lambda k$ and obtain the desired result as follows: 
$$(1-\epsilon) \sqrt{\frac{\rate{C(r)}}{k}} \ge (1-\epsilon) \sqrt{\lambda(1-\rate{k})} > \gamma
$$
where the last inequality follows from the condition on $\gamma$ in \Cref{eq:lambda-LB-cor} and the hypothesis that $\lambda$ satisfies \Cref{eq:linear-ell-cases}. The above inequality and \Cref{le:LB-R} implies the claim.
\end{proof}

\begin{proof}[Proof of \Cref{le:vanishing-corr}]
Both $S(\bx')$ and $S(\bx'')$ are Gaussian random variables $\normal(\mu, \beta^2 \binom{k}{d})$, where $\mu$ is the mean, and depends on the amount of nodes that the solution shares with the planted one $\bx$. The correlation is hence 
$$\rho = \frac{\E[(S(\bx') - \E S(\bx'))(S(\bx'') - \E S(\bx''))]}{\beta^2 \binom{k}{d}}.$$
Observe that $S(\bx') - \E S(\bx')$ and $S(\bx'') - \E S(\bx'')$ are the sum of $\binom{k}{d} - \binom{r}{d}$ independent terms, and $\binom{r}{d}$ identical terms that are Gaussian distributed $\mathcal{N}(0, \beta^2)$. Hence, we get $\rho = \binom{r}{d}/\binom{k}{d}$. We can now upper bound the correlation as
\begin{align*}
k(k-1)\cdots (k-d+1) \geq (k - d)^d &= \left(1 - \frac{d}{k}\right)^d k^d \\ 
&\ge \left(1 - \frac{d^2}{k}\right) k^d
\end{align*}
and therefore 
\begin{align*}
\frac{\binom{\share}{d}}{\binom{k}{d}} = \frac{\share!}{d!(\share-d)!}\frac{d!(k-d)!}{\share !}
&=\frac{\share(\share-1)\cdots (\share-d+1)}{k(k-1)\cdots (k-d+1)} \\ 
&\le \left(\frac{\share}{k}\right)^d \frac{1}{\left(1 - \frac{d^2}{k}\right)} \\
&\leq 
\frac{\lambda^d }{\left(1 - \frac{d^2}{k}\right)} \ .
\end{align*}
For $d \in o(\sqrt{k})$ and $d \in \omega(1)$ the latter quantity converges to $0$ for any constant $\lambda<1$. As for the other case in \Cref{eq:linear-ell-cases}, we have 
\begin{align*}
\frac{\binom{\share}{d}}{\binom{k}{d}} & \leq \frac{e^d \share ^d}{d^d} \cdot  \frac{d^d}{k^d} = \left(\frac{e \cdot \share}{k}\right)^d \leq \left(e \lambda \right)^d \rightarrow 0
\end{align*}
where the asymptotics follows from $e < c_0$ and $d \rightarrow \infty$. As for the second case, since $d \in \Theta(1)$ we have $\binom{\share}{d} \in \Theta(\share^{d})$ and $\binom{k}{d} \in \Theta(k^{d})$,  and thus $\share \in o(k)$ implies $\binom{\share}{d} \in o(\binom{k}{d})$.    
\end{proof}

\begin{proof}[Proof of \Cref{le:max-concentration}]
We apply the tail bound to the maximum of correlated Gaussians given in \cite{lopes2018maximum} (given also in \Cref{le:max-Gaussians} for convenience)
with $N=C(r)$. In particular, since by assumption the correlation vanishes, we can choose \emph{any} $\rho_0 <1$ and, for sufficiently large $p$, satisfy $\delta_0 \sqrt{1- \rho_0} \ge 1- \epsilon$ and $\rho(\share)\leq \rho_0$. We can hence obtain:
\begin{align}
\label{eq:max-concentration}
\lim_{p \to +\infty} \Pr\left(\max_{\hat{\bx} \in \mathcal{C}(r)} S(\hat{\bx}) \le (1- \epsilon) \cdot \sigma_k  \sqrt{2 \log C(r)}\right) = 0\ .
\end{align}
Then, \Cref{eq:max-concentration} follows from \Cref{eq:max-Gaussians} as both $\eta$ and $\xi$ in \Cref{eq:max-Gaussians-constants} are constants and $C(r)=|\mathcal{C}(r)| \rightarrow \infty$. Finally, recall that $\rate{C(r)}= \lim_p \log C(r)/\log p$, and hence, for $p$ large enough, $\log C(r) \le \rate{C(r)} \log p (1+\epsilon)$. The claim follows from the arbitrariness of $\epsilon$.
\end{proof}

\begin{proof}[Proof of \Cref{le:planted-concentration}]
Simply observe that, fixing $\bx$ as the planted solution, $S(\bx) \sim \Gaussian(\binom{k}{d} \beta,{\binom{k}{d}})$. 
Hence, by applying \Cref{le:gaussian_tail} with $\mu = \binom{k}{d} \beta = \beta_k$, $\sigma^2 = \binom{k}{d}$, and $c= t_\Delta =  \sqrt{\Delta 2\log n}$ we have 
\begin{align*}
\Pr_z\left(S(\bx) > \beta_k + t_\Delta \right) &\le \frac{1}{t_\Delta} \cdot \frac{e^{-t^2_\Delta/2\binom{k}{d}}}{\sqrt{2\pi}} 
\nonumber \\
&= \frac{e^{-(\Delta \log p)/\binom{k}{d}}}{t_\Delta \sqrt{2\pi}} \nonumber \\ 
&= \frac{p^{-\Delta/\binom{k}{d}}}{\sqrt{2\pi \Delta} \cdot \sqrt{2 \log p}}  \ ,
\end{align*}
and the latter quantity goes to $0$ for $\Delta = \Omega\left(\frac{\binom{k}{d}}{\log p}\right)$.
\end{proof}

\begin{proof}[Proof of \Cref{le:R-lb}]
We construct $\mathcal{C}(r)$ by a iterative greedy procedure. 
Starting from an arbitrary $\hat{\bx} \in \mathcal{C}_{p,k}$, include $\hat{\bx}$ into $\mathcal{C}(r)$ and remove all solutions in $\mathcal{C}_{p,k}$ with overlap with the planted solution $\langle \bx, \hat{\bx} \rangle \ge r$. Iterate this step with the remaining solutions in $\mathcal{C}_{p,k}$ not considered before, until there are none with this property. At every step we include one new solution in $\mathcal{C}(r)$ we remove at most $B(r)$ solutions from $\mathcal{C}_{p,k}$. Hence, in total we can collect in $\mathcal{C}(r)$ at least $\lceil|\mathcal{C}_{p,k}|/B(r)\rceil = \lceil\binom{p}{k}/B(r) \rceil$ many solutions. Note that this construction satisfies \Cref{def:cover} as (i) holds by construction, and (ii) follows from the fact that if $\hat{\bx} \not \in \mathcal{C}(r)$, then we can perform another greedy step of the procedure and included it. Hence this proves \Cref{eq:cover-ratio}.
To prove the second part of the lemma, note that $\binom{k}{d} \leq 2^k$ and $\binom{p-k}{k-l}\leq \binom{p}{k - r}$, hence we can get the following upper bound on $B(r)$.
\begin{equation*}
B(r) \leq 2^k \binom{p}{k - r}(k - r) \ .
\end{equation*}
Plugging the latter in \Cref{eq:cover-ratio} we get:
\begin{align*}
C(r) \ge \log \frac{\binom{p}{k}}{B(r)} &\geq \log \binom{p}{k} + \\
&\hspace{0.8cm} -\left[ k +\log \binom{p}{k - r}  + \log (k - r)\right]. 
\end{align*}
Using the standard concentration inequalities on the binomial coefficients in \Cref{le:log_binomial}, we get
\begin{align*}
\log \binom{p}{k} \geq k(\log p - \log k) = k (1-\rate{k})\log p
\end{align*}	
and
\begin{align*}
\log \binom{p}{k - r}  \leq& (k-r) (1+\log p - \log (k-r)) \nonumber \\
\lesssim & (k-r) (1 - \rate{k})\log p   \ ,
\end{align*}
where the last inequality follows from $r \leq \lambda k$ and, in particular, $\log (k - r)\geq \log (k-\lambda k) = \log k + \log(1-\lambda)\approx \log k \approx \rate{k} \log p$. Hence we finally get the bound for the rate of the cardinality $C(r)$:
\begin{align*}
\rate{C(r)} &\coloneqq \frac{\log C(r)}{\log p} \\
&\gtrsim k(1-r) - (k -r) (1 - \rate{k}) - \frac{k + \log (k-r)}{\log p} \\ 
&= r(1 - \rate{k}) - o(r) \approx r(1-\rate{k}) \ ,  
\end{align*}
where the last equality is due to the fact that $r \in \omega(\frac{k}{\log p}) $.
\end{proof}

\begin{proof}[Proof of \Cref{th:infoTheoretic}]
Let $\bY^{(u)}$ denote the $\binom{p}{d}$-dimensional vector obtained by the unfolding of the $d$-tensor $\bY$ into a vector containing its non-zero components (all $\bY_{i_1 i_2\cdots i_d}$ for distinct $d$-tuples $i_1<i_2<\cdots < i_d$). Since each component of vector $\bY^{(u)}$ is a Gaussian r.v. according to \Cref{eq:weight-def}, 
the vector $\bY^{(u)}$ is also distributed as 
a Gaussian,
\begin{align}
\label{eq:flat-low}
\bY^{(u)}(\hat{\bx}) \sim \Gaussian \left ( \beta  (\hat{\bx}^{\otimes d})^{(u)} , \id_{\binom{p}{d}} \right ) \coloneqq P_{\hat{\bx}} \
\end{align}
where $\id_{\binom{p}{d}}$ is the identity matrix in $\reals^{\binom{p}{d}}$.
For any $l\in \{0,\ldots,k\}$ and for $r=k-l$, let $\mathcal{C}(r)$ be a maximum-cardinality coverage defined according to \Cref{le:R-lb}. For any $\bx' \in \mathcal{C}_{p,k}$ (possibly $\bx' \not \in \mathcal{C}(r)$) we consider its closest solution in  $\mathcal{C}(r)$ according to the scalar product:
$$\mathcal{P}_{\mathcal{C}(r)}(\bx') \coloneqq \argmax_{\bx'' \in\mathcal{C}(r)} \langle \bx', \bx'' \rangle\ .
$$ 
Let $\tilde \bx \coloneqq \mathcal{P}_{\mathcal{C}(r)}(\bx_{\MLE}(\bY))$, where $\bx_{\MLE}(\bY)$ is the MLE characterized in \Cref{eq:mle_estimator}. By Fano's inequality \cite{cover1999elements} we have that, for $\bx'$ be chosen uniformly at random in $\mathcal{C}(r)$,
\begin{align*}
1-P_r &\ge  \P\left(\tilde \bx \ne \bx'\right) \nonumber \\
&\ge 1 - \frac{{\rm I}(\bx ; \bY)+\log 2}{\log C(r)} 
\end{align*}
where $\rm I(\cdot;\cdot)$ denotes the mutual information.
Using the generalized Fano's inequality \cite{verdu1994generalizing} we further have:
\begin{align*}
{\rm I}(\bx;\bY) \le \frac{1}{C(r)^2}\sum_{\bx'' \ne
\bx' \in \mathcal{C}(r)}D(P_{\bx'} \Vert P_{\bx''}) \le  \binom{k}{d} \beta^2 \ ,
\end{align*}
where $D(\cdot || \cdot)$ denotes the Kullback–Leibler divergence and the second inequality follows from \Cref{le:Kullback}. Hence, from the definition of $\gamma$, we get
\begin{align*}
1 - P_r \ge 1- \frac{\binom{k}{d} \beta^2 + \log
2}{\log C(r)} &\approx 1- \frac{2\gamma^2 k \cdot \log p}{\log C(r)} \nonumber \\
&= 1- \frac{2\gamma^2 k}{\rate{C(r)}}
\end{align*}
where in the last equality we used the definition of $\rate{C(r)}$ in \Cref{eq:rate} and the previous approximation comes from the fact that $C(r) \rightarrow \infty$. 
The bound $\rate{C(r)} \gtrsim r(1-\rate{k})$ in  \Cref{le:R-lb}, with $r = \lambda k$ and $\lambda \in (0,1)$ constant, yields
\begin{align*}
P_r \lesssim 2\gamma^2 \frac{1}{\lambda(1-\rate{k})}
\end{align*}
from which the claim follows easily from the arbitrariness of $\lambda$.
\end{proof}

\section{Useful lemmas}
\begin{lemma}
\label{le:gaussian_tail}
For any $X \sim \Gaussian(\mu,\sigma^2)$ and any $c>0$, the following concentration inequalities hold:
\begin{equation*}
\left(\frac{1}{c} - \frac{1}{c^2}\right)\cdot  \frac{e^{-c^2/2\sigma^2}}{\sqrt{2\pi}} \leq \Pr(X \ge \mu+c) \leq \frac{1}{c}\cdot  \frac{e^{-c^2/2\sigma^2}}{\sqrt{2\pi}}.
\end{equation*}
\end{lemma}
\begin{proof}
See \cite[Section~7.1]{feller2008introduction}.
\end{proof}

\begin{lemma}[Theorem~2.2 in \cite{lopes2018maximum}]
\label{le:max-Gaussians}
For any constant $\delta_0 \in (0,1)$, the maximum of  $N$ possibly dependent Gaussian random variables $X_1,\ldots,X_N \sim \Gaussian(0,\sigma_X^2)$ satisfies
\begin{align}
\label{eq:max-Gaussians}
\Pr\left(\max(X_1,\ldots,X_N) \leq \sigma_X \cdot \delta_0 \sqrt{2(1-\rho_0) \log N}\right) &\le \nonumber \\
&\hspace{-2cm} \le C \cdot \frac{\log^\eta N}{N^\xi}
\end{align}
where  $C=C(\delta_0,\rho_0)$ is a constant depending only on $\delta_0$ and $\rho_0$ and
\begin{align}
\label{eq:max-Gaussians-constants}
\eta = \frac{1 - \rho_0}{\rho_0} (1 - \delta_0) \ , && \text{ and } && \xi = \frac{1 - \rho_0}{\rho_0} (1 - \delta_0)^2 \ .
\end{align}
\end{lemma}

\begin{lemma}
\label{le:inequility_t}
For any random variables $A$ and $B$ and for any $t \in \reals$ the following inequality holds:
\begin{equation*}
\P(A \ge B) \le \P(A > t) + \P(t \ge Y)
\end{equation*}
\end{lemma}
\begin{proof}
Using the conditional probabilities we have
\begin{align*}
\P(A \ge B ) &= \P(A \ge B | A > t) \P(A > t) + \\
&\hspace{1cm} + \P(A \ge B | A \le t) \P(A \le t),
\end{align*}
and using the simple bounds $\P(A \ge B | A > t) \le 1$, $\P(A \le t) \le 1$ and $\P(A \ge B | A \le t) \le \P(B \le t)$ we get the claim.
\end{proof}

\begin{lemma}
\label{le:log_binomial}
For any integers $a$ and $b$ the following concentration inequalities hold for the log binomial:
\begin{align*}
a(\log b - \log a) \leq \log\binom{b}{a} \leq a(1+\log b - \log a).
\end{align*}
\end{lemma}

\begin{lemma}
\label{le:binom_inequality}
For any integers $k$, $d$ and $k'$ with $k' \le k$ the following holds: 
\begin{align*}
\frac{k-{k'}}{\binom{k}{d} - {\binom{k'}{d}}} \le \frac{k}{\binom{k}{d}}. 
\end{align*}
\end{lemma}
\begin{proof}
With simple manipulation, observe that this inequality is equivalent to 
\begin{align*}
\frac{{\binom{k'} {d}}}{\binom{k}{d}}  \leq \frac{{k'}}{k}\ . 
\end{align*}
For $k' < d$ this inequality is trivially satisfied since $\binom{k'}{d}=0$. Otherwise we can write the previous inequality as 
\begin{align*}
\frac{\binom{k'}{d}}{\binom{k}{d}} = \frac{{k'}!}{d!({k'}-d)!}\frac{d!(k-d)!}{k!} 
&=\frac{{k'}({k'}-1)\cdots ({k'}-d+1)}{k(k-1)\cdots (k-d+1)} \\
&\le \frac{{k'}}{k}
\end{align*}
which is satisfied for any ${k'} \leq k$ since all these terms satisfy $\frac{{k'}-i}{k-i} \leq 1$, for $1\leq i \leq d-1$.
\end{proof}

\begin{lemma}
\label{le:UB-concentration}
For any $\Delta>0$, define $t_\Delta \coloneqq \sqrt{2 \Delta \log p}.$ Then the following bound holds for any $\hat \bx \in \mathcal S_{k'}^{F}$:
\begin{equation}
\label{eq:bound-generic:m-intersecting}
\Pr \left(S^{-F}(\hat \bx) \ge t_\Delta \right) \leq p(k',\Delta) \coloneqq \left(\frac{1}{p}\right)^{\frac{\Delta}{D({k'})}} \frac{1}{\sqrt{4 \pi \Delta \log p}}.
\end{equation}
Moreover, the following inequalities hold:
\begin{align}
\label{eq:LB-bound:m-intersecting}
&\P\left(\max_{\hat \bx \in \mathcal{S}_{k'}^{-F}} S^{-F}(\hat \bx) \ge  t_{\Delta} \right) \leq Q(k') \cdot p(k', \Delta) \\
\label{eq:UB-bound:planted}
&\Pr\left(S^{-F}(\bx) < D(k') \beta - t_\Delta \right) \leq p(k', \Delta).
\end{align}
\end{lemma}
\begin{proof}
Observe that each $\hat \bx \in \mathcal{S}^{(-F)}$, $S^{-F}(\hat \bx)$ consists of $D(k')$ non-biased edges, and therefore $S^{-F}(\hat \bx) \sim \Gaussian(0,D(k'))$. 
Hence, by using the tail bound in \Cref{le:gaussian_tail} with $t= t_\Delta$ we have:
\begin{align*}
\Pr\left(S^{-F}(\hat \bx) \ge t_\Delta \right)  \le \frac{1}{t_\Delta}\cdot \frac{e^{-t^2_\Delta/2D(k')}}{\sqrt{2\pi}} 
&= \frac{e^{-\Delta \log p /D(k')}}{t_\Delta \sqrt{2\pi}} \nonumber  \\
&= \frac{p^{-\Delta/D(k')}}{\sqrt{2\pi \Delta} \cdot \sqrt{2 \log p}}
\end{align*}
which proves \Cref{eq:bound-generic:m-intersecting}. By the union bound and \Cref{eq:bound-generic:m-intersecting} we obtain:
\begin{align*}
\Pr\left(\max_{\hat \bx \in \mathcal{S}_{k'}^{-F}} S^{-F}(\hat \bx) \ge t_\Delta \right) =& \Pr\left(\bigcup_{\hat \bx \in \mathcal{S}_{k'}^{-F}} S^{-F}(\hat \bx) \ge t_\Delta \right) \\
& \le Q(k') p(k',\Delta) 
\end{align*}
from which \Cref{eq:LB-bound:m-intersecting} follows noting that $|\mathcal{S}_{k'}^{-F}| = Q(k')$. 
Finally, since all $D(k')$ edges of $S^{-F}(\bx)$ are biased, we have  $S^{-F}(\bx) \sim \Gaussian(D(k') \beta, D(k'))$. Therefore, by using \Cref{le:gaussian_tail} with $t = D(k') \beta + t_\Delta$, we get
\begin{align*}
\P\left(S^{-F}(\bx) \le D(k') \beta - t_\Delta\right) &= \P\left(S^{-F}(\bx) \ge D(k') \beta + t_\Delta\right) \\
& \le \frac{1}{t_\Delta}\cdot \frac{e^{-t^2_\Delta/2D(k')}}{\sqrt{2\pi}}. \end{align*} 
The remainder of the proof is as above. 
\end{proof}

\begin{lemma}
\label{le:existance_t}
For $\gamma > \gamma^{({k'})}_{UB_\epsilon}$, there exists $t$ such that
\begin{equation*}
t_{\Delta'} < t < D(k') \beta - t_{\Delta''}
\end{equation*}
with
\begin{equation}
\frac{\Delta'}{D(k')} = \frac{\log M(k')}{\log p} + \epsilon\ \quad \text{and} \quad \frac{\Delta''}{D(k')} = \frac{\log \binom{k}{k'}}{\log p} + \epsilon \ .
\end{equation}
\end{lemma}
\begin{proof}
We can rewrite the inequality $t_{\Delta'}  < D(k') \beta - t_{\Delta''}$ as follows:
\begin{align*}
D(k') \beta &>  t_{\Delta'} + t_{\Delta''}  = \sqrt{2 \log p}(\sqrt{\Delta'} + \sqrt{\Delta''})  
\end{align*}
hence, by the rescaling in \Cref{eq:rescaling}, we get the condition for the SNR:
\begin{align*}
\frac{\beta}{\sqrt{2 \log p}} &>  \frac{1}{D(k')}(\sqrt{\Delta'} + \sqrt{\Delta''}) \\
&=\sqrt{\frac{1}{D(k')}} \left(\sqrt{\frac{\log M(k')}{\log p}+\epsilon} + \sqrt{\frac{\log \binom{k}{k'}}{\log p}+\epsilon}\right) \\
&= \sqrt{\frac{1}{D(k')}} \cdot UB_{\epsilon}({k'})
\end{align*}
from which the claim follows.
\end{proof}

\begin{lemma}
\label{le:gamma-constants}
For every $k$ and $k'\in\{0,\ldots,k-1\}$, it holds that
\begin{align}
\gamma^{({k'})}_{UB_\epsilon} \le & \left(\sqrt{1+ \rate{k}-2 \rate{k}'+\epsilon} + \sqrt{\rate{k} - \rate{k}'+\epsilon}\right) \times \nonumber \\
&\hspace{3cm} \times \sqrt{\frac{(k-{k'})\binom{k}{d}}{(\binom{k}{d} - \binom{k'}{d}) k }}
\label{eq:lemma_first_part}\\ 
\leq &  \left(\sqrt{1+ \rate{k}-2 \rate{k}'+\epsilon} + \sqrt{\rate{k} - \rate{k}'+\epsilon}\right) \ .
\label{eq:lemma_second_part}
\end{align}
\end{lemma}

\begin{proof}
Using the standard inequalities on binomial coefficients in \Cref{le:log_binomial} and the definition of rate in \Cref{eq:rate} we have
\begin{align*}
\frac{\log \binom{k}{k'}}{\log p} &\le (k-{k'})\frac{1 + \log k - \log (k-{k'})}{\log p} \\
& \approx  (k-{k'}) (\rate{k} - \rate{k - k'}),
\end{align*}
and analogously
\begin{align*}
\frac{\log \binom{p-k}{k -k'}}{\log p} \le &  (k-{k'})\frac{1 + \log (p-k) - \log (k-{k'})}{\log p} \\
& \approx (k-{k'}) (1 - \rate{k - k'}),
\end{align*}
thus implying 
\begin{align*}
\frac{\log M(k')}{\log n} = \frac{\log \left(\binom{k}{k'}\binom{p-k}{k-k'}\right)}{\log p} & \lesssim (k-{k'})(1+ \rate{k} -2 \rate{k - k'}) \ .
\end{align*}
By plugging this into the definition of $UB_\epsilon({k'})$ in \Cref{eq:UB:gamma_k}, we get
\begin{align*}
UB_\epsilon({k'}) &= \sqrt{{\frac{\log M(k')}{\log p}}+\epsilon} + \sqrt{{\frac{\log\binom{k}{k'}}{\log p}}+\epsilon} \\
&\lesssim  
\sqrt{1+ \rate{k} -2 \rate{k-k'} +\epsilon} + \sqrt{\rate{k}  - \rate{k - k'}+\epsilon} \ .
\end{align*}
Hence, using the definition of $\gamma^{({k'})}_{UB_\epsilon}$ given in \Cref{le:UB_epsilon}, we obtain \Cref{eq:lemma_first_part}.
To conclude the proof and have \Cref{eq:lemma_second_part} we finally use the simple manipulations of the \Cref{le:binom_inequality}.
\end{proof}

\begin{lemma}
\label{le:Kullback} 
For any two vectors $\bx',\bx'' \in \mathcal{C}_{p,k}$ we have
\begin{align*}
D(P_{\bx'} \Vert P_{\bx''}) \le \binom{k}{d}\beta^2\ ,
\end{align*}
where $D(\cdot \Vert\cdot)$ denotes
the Kullback-Leiber divergence. 
\end{lemma}
\begin{proof}
Since $P_{\bx'}$ and $P_{\bx''}$ are a Gaussian probability distributions \eqref{eq:flat-low}, we have
\begin{align*}
D(P_{\bx'} \Vert P_{\bx''}) &= \frac{1}{2} \beta^2  \|(\bx'\od)^{(u)} - (\bx''\od)^{(u)}\|^2 \\
&= \beta^2 ( \|(\bx'\od)^{(u)} \|^2 
- \<(\bx'\od)^{(u)},(\bx''\od)^{(u)}\> ) 
\\
&\le \beta^2 \|(\bx'\od)^{(u)} \|^2 =  \beta^2\binom{k}{d}\ .
\end{align*}
\end{proof}

\section{Derivations for \texorpdfstring{\Cref{sec:amp}}{} (AMP algorithm)}
\label{app:amp}
\subsection{Derivation of the AMP iterative equations}
In the following we adopt conveniently the notation used in \cite{lesieur2017statistical}, denoting as $$\bm{S} = \beta \sqrt{\binom{p-1}{d-1}} \bm{Y}.$$ 
Throughout this section, we shall make the following assumption on the above rescaled tensor, i.e., on how $\beta$ scales with respect to the entries of $\bm{Y}$. \begin{assumption}
\label{asm:less-one-amp}
For every $d$-tuple  $i_1i_2\cdots i_d$ of distinct indices, it holds that $S_{i_1i_2\cdots i_d} \le 1$, that is, $\beta \le 1/{\sqrt{\binom{p-1}{d-1}}
Y_{i_1i_2\cdots i_d}
}
$. 
\end{assumption}
A stronger assumption is the following one:
\begin{assumption}
\label{asm:vanish-amp}
For every $d$-tuple  $i_1i_2\cdots i_d$ of distinct indices, it holds that $S_{i_1i_2\cdots i_d} \in o(1)$.
\end{assumption}

We next present a message passing algorithm  which is a simple generalization of the algorithm for matrix \cite{lesieur2017constrained} and tensor-PCA \cite{lesieur2017statistical}. The algorithm is described by the following iterative equations:

\begin{align*}
&x_{i\to i i_2 \dots i_d}^{(t)} = \frac{1}{\sqrt{\binom{p-1}{d-1}}} \sum_{\substack{k_2 < \dots < k_d \\ (k_2, \dots, k_d) \neq (i_2, \dots, i_d)}} \big(S_{i k_2 \dots k_d} \cdot \\
&\hspace{4cm} \cdot \hat{x}^{(t)}_{k_2 \to i k_2 \dots k_d} \cdot \ldots \cdot \hat{x}^{(t)}_{k_d \to i k_2 \dots k_d} \big) \\
&A_{i\to i i_2 \dots i_d}^{(t)} = \frac{1}{\binom{p-1}{d-1}} \sum_{\substack{k_2 < \dots < k_d \\ (k_2, \dots, k_d) \neq (i_2, \dots, i_d)}} \big( S^2_{i k_2 \dots k_d} \cdot \\
&\hspace{3cm} \cdot (\hat{x}^{(t)}_{k_2 \to i k_2 \dots k_d})^2 \cdot \ldots \cdot (\hat{x}^{(t)}_{k_d \to i k_2 \dots k_d})^2 \big)  \\
&\left(\hat{x}_{i \to i i_2 \dots i_d}^{(t+1)}\right)_i = f\left( (A_{ i \to i i_2 \dots i_d}^{(t)})_i, (x_{ i \to i i_2 \dots i_d}^{(t)})_i\right) 
\end{align*}
where $f(\cdot, \cdot)$ is the multidimensional threshold function defined in \Cref{eq:threshold_f_multidimensional} and $\left(\hat{x}_{i \to i i_2 \dots i_d}^{(t+1)}\right)_i$ denotes the $p$-dimensional vector indexed by $i$. Note that the function $f$ is applied element-wise on all indices other than $i$, that is, on $i_2, \ldots, i_d$. Since the messages depend weakly on the target factor, we can compute the messages including the factor $(i_2,\ldots,i_d)$ in the sums above, and dropping the factor node dependency as:
\begin{align}
&x_{i}^{(t)} = \frac{1}{\sqrt{\binom{p-1}{d-1}}} \sum_{i_2 < \dots < i_d} \big(S_{i i_2 \dots i_d} \cdot \nonumber \\
&\hspace{3cm} \cdot \hat{x}^{(t)}_{i_2 \to i i_2 \dots i_d} \cdot \ldots \cdot \hat{x}^{(t)}_{i_h \to i i_2 \dots i_d} \big)
\label{eq:amp_messages_x_i}\\
&A_{i}^{(t)} = \frac{1}{\binom{p-1}{d-1}} \sum_{i_2 < \dots < i_d} \big( S^2_{i i_2 \dots i_d} \cdot \nonumber \\
&\hspace{2cm} \cdot (\hat{x}^{(t)}_{i_2 \to i i_2 \dots i_d})^2 \cdot \ldots \cdot (\hat{x}^{(t)}_{i_d \to i i_2 \dots i_d})^2 \big)
\label{eq:amp_messages_A_i}\\
&\bm{\hat{x}}^{(t+1)} = f\left( \bm{A}^t, \bm{x}^t\right).  \nonumber
\end{align}
We can now analyze the error obtained by this simplification as:   
\begin{align}
&x_i^{(t)} - x_{i\to i i_2 \dots i_h}^{(t)} = \frac{1}{\sqrt{\binom{p-1}{d-1}}}\big( S_{i i_2 \dots i_d} \nonumber\\
&\hspace{3.2cm} \cdot \hat{x}^{(t)}_{i_2 \to i i_2 \dots i_d} \cdot \ldots \cdot \hat{x}^{(t)}_{i_d \to i i_2 \dots i_d} \big)
\label{eq:amp_x_error}\\
&A_i^{(t)} - A^{(t)}_{i\to i i_2 \dots i_d} = \frac{1}{\binom{p-1}{d-1}} \big(S^2_{i i_2 \dots i_d} \nonumber \\ 
&\hspace{2.2cm} \cdot (\hat{x}^{(t)}_{i_2 \to i i_2 \dots i_d})^2 \cdot \ldots \cdot (\hat{x}^t_{i_d \to i i_2 \dots i_d})^2\big).
\label{eq:amp_A_error}
\end{align}
Observe that by \Cref{asm:less-one-amp}, the error in $A$ is of lower order then the one in $x$.
Hence, we can estimate the error in $\bm{\hat{x}}^{t+1}$ by focusing on the quantity in \Cref{eq:amp_x_error}:
\begin{align}
&\hat{x}_{i\to i i_2 \dots i_d}^{(t)} - \hat{x}_i^{(t)} = f\left( (A^{(t-1)}_{ i \to i i_2 \dots i_d})_i, (x_{ i \to i i_2 \dots i_d}^{(t-1)})_i\right)_i + \nonumber\\ 
&\hspace{5cm} - f(\bm{A}^{(t-1)}, \bm{x}^{(t-1)})_i  \nonumber \\ 
& = - \partial_{x_i} f(\bm{A}^{(t-1)}, \bm{x}^{(t-1)})_i \frac{1}{\sqrt{\binom{p-1}{d-1}}} \big( S_{i i_2 \dots i_d} \cdot \nonumber \\
& \hspace{1cm} \cdot x_{i_2\to i i_2 \dots i_d}^{(t-1)} \cdot \ldots \cdot x_{i_d\to i i_2 \dots i_d}^{(t-1)}\big) + o\left(\frac{S_{i i_2 \dots i_d}}{\sqrt{\binom{p-1}{d-1}}}\right) \nonumber \\
& = - \partial_{x_i} f(\bm{A}^{(t-1)}, \bm{x}^{(t-1)})_i \frac{1}{\sqrt{\binom{p-1}{d-1}}} \big( S_{i i_2 \dots i_d} \cdot \nonumber\\
&\hspace{2cm} x_{i_2}^{(t-1)} \cdot \ldots \cdot x_{i_d}^{(t-1)} \big) + o\left(\frac{S_{i i_2 \dots i_d}}{\sqrt{\binom{p-1}{d-1}}}\right) \ .
\label{eq:amp_expansion_x_hat}
\end{align}
Plugging this last expansion into \Cref{eq:amp_messages_x_i}, and by recalling the definition of $\bm{\sigma}^{(t)}$ in \Cref{eq:amp_algorithm}, we get to leading order:
\begin{multline*}
x_i^{(t)} = \frac{1}{\sqrt{\binom{p-1}{d-1}}} \sum_{i_2 < \dots < i_d} S_{i i_2 \dots i_d} \hat{x}_{i_2}^{(t)} \cdot \ldots \cdot \hat{x}_{i_d}^{(t)} + \\
- \frac{d-1}{\binom{p-1}{d-1}} \sum_{i_2} \sigma_{i_2}^{(t)} \sum_{i_3 < \dots < i_d} S_{i i_2 \dots i_d}^2  \hat{x}_{i_3}^{(t)} \hat{x}_{i_3}^{(t-1)} \cdot \ldots \cdot \hat{x}_{i_d}^{(t)} \hat{x}_{i_d}^{(t-1)} + \\
\hspace{3cm} + o\left(\frac{S_{i i_2 \dots i_d}^2}{\binom{p-1}{d-1}}\right)
\end{multline*}
where the factor $d-1$ comes from the symmetry of choosing the index ($i_2$ in our case) for the lower order term in the product.  Note that all the other terms are of lower order and are hence discarded. 
Plugging \Cref{eq:amp_expansion_x_hat} into \Cref{eq:amp_messages_A_i} we obtain:
\begin{align*}
A_{i}^{(t)} = \frac{1}{\binom{p-1}{d-1}} \sum_{i_2 < \dots < i_d} S^2_{i i_2 \dots i_d} \cdot (\hat{x}^{(t)}_{i_2})^2 \cdot \ldots \cdot (\hat{x}^{(t)}_{i_d})^2 + \\
\hspace{3cm} + o\left(\frac{S_{i i_2 \dots i_d}^2}{\binom{p-1}{d-1}}\right) \ .
\end{align*}
\subsection{Derivation of the state evolution}
We assume here a typical condition for belief propagation algorithm, that is the statistical independence between the rescaled observation $\bm{S}$ and the estimates $\hat{\bx}$ (see for reference \cite{montanari2013statistical, lesieur2017constrained}). Given the independence, the central limit theorem applies for the right hand side of both \Cref{eq:amp_messages_x_i} and \Cref{eq:amp_messages_A_i}, so that the messages behave in the large $p$ limit as Gaussian random variables. We can hence describe the evolution of the AMP algorithm tracking only the average and variance of such messages (recall that $\bx$ is the planted solution): 
\begin{align}
\E_{\bZ} [x_i^{(t)}] &= \frac{1}{\sqrt{\binom{p-1}{d-1}}} \sum_{i_2 < \dots < i_d} \big( \E_{\bZ}[S_{i i_2 \dots i_d}] \cdot \nonumber\\
&\hspace{3cm} \cdot \hat{x}^{(t)}_{i_2 \to i i_2 \dots i_d} \cdot \ldots \cdot \hat{x}^{(t)}_{i_d \to i i_2 \dots i_d} \big)   \nonumber\\
&= \beta^2 \sum_{i_2 < \dots < i_d} \big( x_i \cdot x_{i_2} \cdot \ldots \cdot x_{i_d} \cdot \nonumber\\
&\hspace{3cm} \cdot \hat{x}^{(t)}_{i_2 \to i i_2 \dots i_d} \cdot \ldots \cdot \hat{x}^{(t)}_{i_d \to i i_2 \dots i_d} \big)   \nonumber\\
&=  \beta^2 \sum_{i_2 < \dots < i_d} x_i \cdot x_{i_2} \cdot \ldots \cdot x_{i_d} \cdot \hat{x}^{(t)}_{i_2} \cdot \ldots \cdot \hat{x}^{(t)}_{i_d} + \nonumber\\
&\hspace{4cm} + o(\beta^2)  \nonumber\\
&= \beta^2 x_i \cdot \bm{1}\{\bm{x} \circ \bm{\hat{x}}^{(t)}\}_i + o(\beta^2) \ .
\label{eq:se_expectation_x}
\end{align}
Analogously 
\begin{align}
\mathbb{V}_{\bm{Z}} [x_i^{(t)}] &= \frac{1}{\binom{p-1}{d-1}} \sum_{i_2 < \dots < i_d} \big(  \mathbb{V}_{\bm{Z}}[S_{i i_2 \dots i_d}] \cdot \nonumber \\
&\hspace{2cm} \cdot (\hat{x}^{(t)}_{i_2 \to i i_2 \dots i_d})^2 \cdot \ldots \cdot (\hat{x}^{(t)}_{i_d \to i i_2 \dots i_d})^2 \big)  \nonumber\\
&= \beta^2 \sum_{i_2 < \dots < i_d} (\hat{x}^{(t)}_{i_2 \to i i_2 \dots i_d})^2 \cdot \ldots \cdot (\hat{x}^{(t)}_{i_d \to i i_2 \dots i_d})^2   \nonumber\\
&=  \beta^2 \sum_{i_2 < \dots < i_d} (x_{i_2}^{(t)})^2 \cdot \ldots \cdot (x_{i_d}^{(t)})^2 + o(\beta^2)   \nonumber\\
&= \beta^2 \bm{1}\{\bm{\hat{x}}^t \circ \bm{\hat{x}}^t\}_i + o(\beta^2)
\label{eq:se_variance_x}
\end{align}
and 
\begin{align}
\E_{\bm{Z}} [A_i^{(t)}] &= \frac{1}{\binom{n-1}{d-1}} \sum_{i_2 < \dots < i_d} \big( \E_{\bm{Z}}[S_{i i_2 \dots i_d}^2] \cdot \nonumber \\
&\hspace{2cm} \cdot (\hat{x}^{(t)}_{i_2 \to i i_2 \dots i_d})^2 \cdot \ldots \cdot (\hat{x}^{(t)}_{i_d \to i i_2 \dots i_d})^2 \big)  \nonumber\\
&= \beta^2 \sum_{i_2 < \dots < i_d} (\beta^2 (x_i)^2 \cdot (x_{i_2})^2 \cdot \ldots \cdot (x_{i_d})^2 + 1) \nonumber\\ 
&\hspace{2cm} \cdot (\hat{x}^{(t)}_{i_2})^2 \cdot \ldots \cdot (\hat{x}^{(t)}_{i_d})^2 + o(\beta^2)   \nonumber\\
&= \beta^4 x_i \cdot \bm{1}\{\bm{x} \circ \bm{x} \circ \bm{\hat{x}}^{(t)} \circ \bm{\hat{x}}^t \}_i + \nonumber\\
&\hspace{2cm} + \beta^2 \bm{1}\{\bm{\hat{x}}^{(t)} \circ \bm{\hat{x}}^{(t)}\}_i + o(\beta^2) \nonumber\\
&= \beta^2 \bm{1}\{\bm{\hat{x}}^{(t)} \circ \bm{\hat{x}}^{(t)}\}_i + o(\beta^2)
\label{eq:se_expectation_A}
\end{align}
where in the second equality we used the fact that \begin{align*}
\E_{\bm{Z}}[S_{i i_2 \dots i_h}^2] =& \beta^2 (x_i)^2 \cdot (x_{i_2})^2 \cdot \ldots \cdot (x_{i_h})^2 + \\
&\hspace{-0.5cm} + \E_{\bm{Z}}[Z_{i i_2 \dots i_h}^2] + \beta x_i \cdot x_{i_2} \cdot \ldots \cdot x_{i_h} \E_{\bm{Z}}[Z_{i i_2 \dots i_h}] \\ =&  \beta^2 (x_i)^2 \cdot (x_{i_2})^2 \cdot \ldots \cdot (x_{i_h})^2 + 1.
\end{align*}
We now assume that the signal estimates $\bm{\hat{x}}^t$ are drawn from the true intractable posterior distribution $\P(\cdot| \bm{x}, \bm{Z})$. Given this assumption and using the Nishimori condition (see \cite{iba1999nishimori, lesieur2017statistical}) we obtain easily from \Cref{eq:se_variance_x} and \Cref{eq:se_expectation_A}: 
\begin{align*}
\mathbb{V}_{\bm{x}, \bm{Z}} [\bm{\hat{x}}^{(t)}] = \E_{\bm{x}, \bm{Z}} [\bm{A}^{(t)}] &= \beta^2 \E_{\bm{x}} \bm{1}\{\bm{\hat{x}}^{(t)} \circ \bm{\hat{x}}^{(t)}\} + o(\beta^2) \\
&=\beta^2 \E_{\bm{x}} \bm{1}\{\bm{x} \circ \bm{\hat{x}}^{(t)}\} + o(\beta^2) \\
&= \bm{\hat{m}}^{(t)} + o(\beta^2). 
\end{align*}
It can be easily seen that the \emph{variance} of the messages $A$ is of lower order, hence $\bm{A}^{(t)}$ can be approximated by only its mean. With the assumed gaussianity of messages, we can hence write using \Cref{eq:se_expectation_x} and \Cref{eq:se_variance_x} $\bm{x}^{(t)} = \bm{\hat{m}}^{(t)} \circ \bm{x} + \sqrt{\bm{\hat{m}}^{(t)}}  \circ \bm{z}$, with $\bm{z}$ being a $p$-dimensional standard Gaussian vector. The state evolution finally reads:
\begin{align*}
\bm{m}^{t+1} &= \frac{1}{\binom{p-1}{d-1}} \E_{\bm{x}, \bm{z}}[\bm{1}\{\bm{x} \circ \bm{\hat{x}}^{t+1}\}] \\
&=\frac{1}{\binom{p-1}{d-1}} \E_{\bm{x}, \bm{z}}[\bm{1}\{\bm{x} \circ f(\bm{\hat{m}}^t, \bm{\hat{m}}^t \circ \bm{x} + \sqrt{\bm{\hat{m}}^t} \circ \bm{z})\}] \ .
\end{align*}
\subsection{Analytical threshold for AMP recovery}
To get an analytical threshold, we start from the SE for factorizable prior as in \Cref{eq:se_factorized} and we study the fixed point of the recursive equation for the parameter of the Bernulli distribution $\delta = k/p \to 0$ in the large system limit. In particular the threshold function reads in the limit $f(a, x) = \delta e^{x - a/2} + O(\delta^2)$, hence the factorized SE becomes $m_{t+1} = \delta^2 \E_z [e^{\hat{m}_t/2 + \sqrt{\hat{m}_t} z}] = \delta^2 e^{\hat{m}_t}$ with $\hat{m}_t = \beta^2 \binom{p-1}{d-1} m_t$. It is easy to see that the critical bias $\beta$ to have a perfect overlap with the planted signal is
$$\beta_{\AMP} \coloneqq \sqrt{\frac{1}{e (d-1)} \frac{p^{2(h-1)}}{\binom{p-1}{d-1}} \frac{1}{k^{2(d-1)}}}.$$ 
Using the definition of normalized snr $\gamma$ in \Cref{eq:rescaling}, we obtain the threshold $\gamma_{\AMP}$ in \Cref{claim:amp_threshold}.

\begin{remark}[Validity of the AMP approximations] We can observe from the derivations above that the AMP equations and state evolution are carried out assuming the quantity $$\frac{S_{i_1, \dots i_h}}{\sqrt{\binom{p-1}{d-1}}} = \beta \sqrt{\binom{p-1}{d-1}} Y_{i_1, \dots i_d}$$ being small (which corresponds to \Cref{asm:vanish-amp} above). This is the case for the classic tensor-PCA for both dense and sparse signal with linear sparsity as in \cite{lesieur2017statistical}. However in the scenario here considered the effective sparsity of the problem defined by the parameter $\delta = k/p$ can be sub-linear, and hence non-trivial estimation requires a $\beta$ (hence an snr) such that the quantity above is not in $o(1)$. For this reason the derivations have to be considered non-rigorous in the regime used in this analysis. The rigorous presentation of AMP-like algorithms in effectively sub-linear sparse estimation problem, also in line of the approach proposed very recently in \cite{barbier2020all}, is left for future developments.
\end{remark}

\section{Comparison with Bounds in the literature}
\subsection{Tensor-PCA formulation}
\label{sec:TPCA}
We follow the terminology of \cite{hopkins2015tensor} to explain the similarities/differences between tensor-PCA and our problem.

\begin{definition}[Symmetric Tensor-PCA]\label{sec:tensor-PCA-formulation}
Given an input tensor $\bY=\tau \cdot \bv^{\otimes d} +\bZ$, where $\bv \in \reals^p$ is an arbitrary \emph{unit} vector, $\tau\geq 0$ is the signal-to-noise ratio, and $\bZ$ is a random noise tensor with iid standard Gaussian entries, recover the signal $v$ approximately. Moreover, the noise tensor is \textit{symmetric} and thus so is the input tensor as well, that is, $\bZ_{\indperm} = \bZ_{\ind}$ and  $\bY_{\indperm} = \bY_{\ind}$ for any permutation $\pi$.
\end{definition}	

\begin{definition}[Planted $k$-Densest Sub-Hypergraph] 
This problem is a variant of symmetric tensor-PCA in which we impose the following additional structure:
\begin{enumerate}
\item\label{cond:distinct}We consider only the $\binom{p}{d}$ entries with \emph{distinct} indices, that is, $\bY_{\ind}=0$ whenever $i_a=i_b$ for some $a$ and $b$.
\item The vector $\bv$ encodes a planted $k$-Subgraph and thus has \emph{exactly} $k$ entries equal to $1/\sqrt{k}$, and all other entries are equal to $0$.
\end{enumerate}
\end{definition}

\begin{remark}[Impact of diagonal entries -- \Cref{cond:distinct}]\label{rem:cond:distinct}
Dropping Condition~\ref{cond:distinct} leads to the variant in which we make the substitution $\binom{p}{d} \mapsto p^d$ in \Cref{cond:distinct} above.
\end{remark}

\begin{remark}[Rescalings]
Different papers consider different rescaling of the signal-to-noise ratio, that are here reported for convenience in \Cref{tab:snr}.
In general, consider a tensor 
\begin{align}
\label{eq:SNR-generic}
\bY =  \mu \cdot \bv^{\otimes d} +\bZ && \bZ \sim \Gaussian(0,\sigma^2)
\end{align}
where $\bv$ is a vector of unit length, and $\mu$ and $\sigma^2$ (signal and noise, respectively) determine the snr. By simply rescaling so that we have normally distributed Gaussian noise, this is the same as 
\begin{align*}
\bY =  \mu/\sigma \cdot \bv^{\otimes d} +\bZ && \bZ \sim \Gaussian(0,1)
\end{align*}
and since in our formulation (snr in \Cref{eq:rescaling}) we consider the planted solution as a 0-1 vector $\bx$ consisting of $k$ ones and $p-k$ zeros, we are effectively considering a planted signal $\beta \cdot \bx^{\otimes d} = \beta \sqrt{k^d} \cdot \bv^{\otimes d}$ where $\bv = \bx/\sqrt{k}$ is a unitary vector. Therefore, the tensor-PCA formulation in \Cref{eq:SNR-generic}
corresponds to 
\begin{align*}
\beta \sqrt{k^d}= (\mu/\sigma) && \Leftrightarrow && \gamma = (\mu/\sigma) \sqrt{\frac{\binom{k}{d}}{k^{d+1}}\cdot \frac{1}{2\log p}}
\end{align*}
We use this relation to convert the existing bounds for tensor-PCA in the literature to our snr $\gamma$ as shown in \Cref{tab:snr}. In [$\star$] we ignore the diagonal entries (see \Cref{rem:cond:distinct}). As we are implicitly considering the easier problem with the additional entries, the bounds that one obtains are in a sense ``optimistic'' for our original problem.
\end{remark}

\begin{table}[ht]
\caption{Signal to noise ratio scaling  in the literature}
\centering\ra{1.3}
\begin{tabular}{lll}
\toprule
\bfseries Tensor & \bfseries Noise & \bfseries snr ($\star$=ours)\\
\midrule
$\bY=\beta \cdot \bv^{\otimes d} +\bZ$ & 
$Z_{\ind} \sim  \Gaussian(0,1/(p(d-1)!))$ & $\beta$~\cite{richard2014statistical} \\
$\bY=\beta' \cdot \bv^{\otimes d} +\bZ$ & 	$Z_{\ind} \sim \Gaussian(0,2/(p\cdot d!))$ & $\beta'$~\cite{montanari2016limitation} \\
$\bY=\beta'' \cdot \bv^{\otimes d} +\bZ$ & 	$Z_{\ind} \sim \Gaussian(0,2/(p\cdot d!))$ & $\beta''$~\cite{perry2020} \\
$\bY=\tau \cdot \bv^{\otimes d} +\bZ$ & 	$Z_{\ind} \sim \Gaussian(0,1)$ & $\tau$~\cite{hopkins2015tensor} \\
$\bY=\lambda \sqrt{p}\cdot \bv^{\otimes d} +\bZ$ & $Z_{\ind} \sim \Gaussian(0,1)$ & $\lambda$~\cite{jagannath2020statistical,arous2020algorithmic} \\
$\bY = \sqrt{\lambda_p} \cdot \bv^{\otimes d} +\bZ$ & $Z_{\ind} \sim \Gaussian(0,1)$ & $\lambda_p$ \cite{niles2020all}\\ 
$\bY=\beta \sqrt{k^d} \cdot \bv^{\otimes d} +\bZ$ & $Z_{\ind} \sim \Gaussian(0,1)$ & 
$\gamma$ \cref{eq:rescaling} $\star$ \\
\bottomrule
\end{tabular}
\label{tab:snr}
\end{table}

\noindent
The information-theoretic bounds translated in our scale read as follows:
\begin{align*}
&\text{lower bound  \cite{richard2014statistical}:} \quad
\gamma \leq \sqrt{\frac{p\cdot d! \binom{k}{d}}{k^{d+1} }\frac{1}{20\log p}}\\
&\text{upper bound  \cite{richard2014statistical}:} \quad 
\gamma \geq \sqrt{\frac{p\cdot d! \binom{k}{d}}{k^{d+1} }\frac{\log d}{2\log p}}
\end{align*}
Sharper bounds have been obtained for detection and (weak) recoverability:
\begin{align*}
&\text{generic spherical prior \cite{perry2020}:} \quad
\gamma = \sqrt{\frac{p\cdot d! \binom{k}{d}}{k^{d+1} }\frac{\log d}{2\log p}}\\
&\text{Radamacher prior \cite{perry2020}:} \quad 
\gamma = \sqrt{\frac{p\cdot d! \binom{k}{d}}{k^{d+1} }\frac{\log d}{4\log p}}
\end{align*}
where the Radamacher prior bounds apply to one of the following restrictions: (i) the dense regimes with any sparsity constant $\rho \in (0,1]$ and $d \rightarrow \infty$ or (2) the vanishing sparsity regime $\rho \rightarrow 0$ and constant $d$. Note that in both cases \cite{richard2014statistical} and \cite{perry2020}, the bound are located at a scale $\sqrt{\frac{p}{k \log p}}$ that diverges for any rate $\rate{k} < 1$ considered in this paper. From this result, we can observe that the recovery above the $\gamma_{\UB}$ given in this paper is possible only thanks to the exploitation of the prior constraint, and it is not possible in general.

Sharp bounds on the MMSE estimator have been also obtained \cite{niles2020all} with the same Bernoulli prior considered here, and translate into the threshold 	
$$\gamma_{\MMSE} = \sqrt{1-\rate{k}}$$
which can be obtained by a proper rescale of the  planted vector $\bv$ so that the resulting tensor (without diagonal entries) has unit length.
Note that these bounds on the MMSE regard the problem of finding a vector with a positive non-vanishing correlation with the planted vector (weak recovery), and correspond to the lower bound $\gamma_{\LB}$ for the MLE provided here in the case of $d \to \infty, d \in o(k)$.

Algorithmic upper bounds provided by sum-of-squares (SOS) algorithms \cite{hopkins2015tensor} are:
\begin{align}
\gamma_{\SOS} \coloneqq \sqrt{\frac{p^{d/2}\binom{k}{d}}{k^{d-1}}\frac{1}{2 \log^{1/2} p}} 
\end{align}
for any $d \ge 3$. 
Note that also for the computational threshold, this bound is higher than the AMP threshold $\gamma_{\AMP} \approx \sqrt{p^{(1-\rate{k})(d-1)}}$ for $$\rate{k} > 1/2 - 1/d.$$

Further papers provide general bounds whose thresholds do \emph{not} have a closed form and apply to the easier problem of detection or hypothesis testing:
\begin{align*}
&\text{\cite{montanari2016limitation}} \quad \beta^{2}_d \coloneqq \inf_{q\in(0,1)} \sqrt{-\frac{1}{q^d}\log (1-q^2)} \\
&\text{\cite{jagannath2020statistical}} \quad \lambda_c \coloneqq \sup_{\lambda\geq 0}\left\{\sup_{t\in[0,1)} f_\lambda(t) \leq 0\right\} 
\end{align*}
with $f_\lambda(t) = \lambda^2 t^d + \log(1-t) + t $

\subsection{Prior bounds for the \texorpdfstring{$k$}{} densest subhypergraph}
\label{sec:prior_bounds_from_us}
We report the prior upper and lower bounds on the very same problem in \cite[Theorem 5]{corinzia2019}. For the sake of comparison, we rewrite the upper and lower bound there according to our scaled-normalized snr $\gamma$, and denote these bounds as $\gamma_{lb}$ and $\gamma_{ub}$, respectively. As we can see below, these bounds are very loose in most of the cases, $\gamma_{lb} \ll \gamma_{LB} \leq \gamma_{UB} \ll \gamma_{ub}$: 
\begin{align*}
\label{eq:old-bounds}
\gamma_{lb} = \sqrt{\frac{1}{d}}
\end{align*}
and
\begin{align*}
\gamma_{ub} = 
\begin{cases}
\sqrt{2} \ & \frac{\binom{k}{d}}{k}\frac{1}{\log p} \to 0  \\
2 \sqrt{1+c(1+\log 2)} \ & \frac{\binom{k}{d}}{k} \frac{1}{\log p} \to c\in (0,+\infty) \\
2 \sqrt{\frac{\binom{k}{d}}{k\log p}\cdot \frac{1 + \log 2}{1-\rate{k}}} \ & \rate{k} \in (0,1)
\end{cases}
\end{align*}
For instance, when $1 \ll d \ll k$ we have $\gamma_{lb} \to 0$ and $\gamma_{ub} \to +\infty$.

\end{appendices}

\clearpage
\bibliographystyle{IEEEtran}
\bibliography{bib}

\begin{thebibliography}{10}
\providecommand{\url}[1]{#1}
\csname url@samestyle\endcsname
\providecommand{\newblock}{\relax}
\providecommand{\bibinfo}[2]{#2}
\providecommand{\BIBentrySTDinterwordspacing}{\spaceskip=0pt\relax}
\providecommand{\BIBentryALTinterwordstretchfactor}{4}
\providecommand{\BIBentryALTinterwordspacing}{\spaceskip=\fontdimen2\font plus
\BIBentryALTinterwordstretchfactor\fontdimen3\font minus
  \fontdimen4\font\relax}
\providecommand{\BIBforeignlanguage}[2]{{%
\expandafter\ifx\csname l@#1\endcsname\relax
\typeout{** WARNING: IEEEtran.bst: No hyphenation pattern has been}%
\typeout{** loaded for the language `#1'. Using the pattern for}%
\typeout{** the default language instead.}%
\else
\language=\csname l@#1\endcsname
\fi
#2}}
\providecommand{\BIBdecl}{\relax}
\BIBdecl

\bibitem{corinzia2019}
L.~Corinzia, P.~Penna, L.~Mondada, and J.~M. Buhmann, ``Exact recovery for a
  family of community-detection generative models,'' in \emph{IEEE
  International Symposium on Information Theory (ISIT)}.\hskip 1em plus 0.5em
  minus 0.4em\relax IEEE, 2019, pp. 415--419.

\bibitem{gu2017functional}
S.~Gu, M.~Yang, J.~D. Medaglia, R.~C. Gur, R.~E. Gur, T.~D. Satterthwaite, and
  D.~S. Bassett, ``Functional hypergraph uncovers novel covariant structures
  over neurodevelopment,'' \emph{Human brain mapping}, vol.~38, no.~8, pp.
  3823--3835, 2017.

\bibitem{wang2012naive}
Z.~Wang, J.~Liu, N.~Zhong, Y.~Qin, H.~Zhou, J.~Yang, and K.~Li, ``A naive
  hypergraph model of brain networks,'' in \emph{International Conference on
  Brain Informatics}.\hskip 1em plus 0.5em minus 0.4em\relax Springer, 2012,
  pp. 119--129.

\bibitem{zu2016identifying}
C.~Zu, Y.~Gao, B.~Munsell, M.~Kim, Z.~Peng, Y.~Zhu, W.~Gao, D.~Zhang, D.~Shen,
  and G.~Wu, ``Identifying high order brain connectome biomarkers via learning
  on hypergraph,'' in \emph{International Workshop on Machine Learning in
  Medical Imaging}.\hskip 1em plus 0.5em minus 0.4em\relax Springer, 2016, pp.
  1--9.

\bibitem{legenstein2018long}
R.~Legenstein, W.~Maass, C.~H. Papadimitriou, and S.~S. Vempala, ``{Long Term
  Memory and the Densest K-Subgraph Problem},'' in \emph{9th Innovations in
  Theoretical Computer Science Conference (ITCS)}, ser. LIPIcs, vol.~94, 2018,
  pp. 57:1--57:15.

\bibitem{jolion2012graph}
J.-M. Jolion and W.~Kropatsch, \emph{Graph based representations in pattern
  recognition}.\hskip 1em plus 0.5em minus 0.4em\relax Springer Science \&
  Business Media, 2012, vol.~12.

\bibitem{Benson163}
\BIBentryALTinterwordspacing
A.~R. Benson, D.~F. Gleich, and J.~Leskovec, ``Higher-order organization of
  complex networks,'' \emph{Science}, vol. 353, no. 6295, pp. 163--166, 2016.
  [Online]. Available:
  \url{https://science.sciencemag.org/content/353/6295/163}
\BIBentrySTDinterwordspacing

\bibitem{grilli2017higher}
J.~Grilli, G.~Barab{\'a}s, M.~J. Michalska-Smith, and S.~Allesina,
  ``Higher-order interactions stabilize dynamics in competitive network
  models,'' \emph{Nature}, vol. 548, no. 7666, pp. 210--213, 2017.

\bibitem{ke2019community}
Z.~T. Ke, F.~Shi, and D.~Xia, ``Community detection for hypergraph networks via
  regularized tensor power iteration,'' \emph{arXiv preprint arXiv:1909.06503},
  2019.

\bibitem{niles2020all}
J.~Niles-Weed and I.~Zadik, ``The all-or-nothing phenomenon in sparse tensor
  pca,'' \emph{Advances in Neural Information Processing Systems}, vol.~33,
  2020.

\bibitem{abbe2016exact}
E.~Abbe, A.~S. Bandeira, and G.~Hall, ``Exact recovery in the stochastic block
  model,'' \emph{IEEE Transactions on Information Theory}, vol.~62, no.~1, pp.
  471--487, 2016.

\bibitem{mossel2015consistency}
E.~Mossel, J.~Neeman, and A.~Sly, ``Consistency thresholds for the planted
  bisection model,'' in \emph{47th ACM Symposium on Theory of Computing
  (STOC)}.\hskip 1em plus 0.5em minus 0.4em\relax ACM, 2015, pp. 69--75.

\bibitem{chen2014statistical}
Y.~Chen and J.~Xu, ``Statistical-computational phase transitions in planted
  models: The high-dimensional setting,'' in \emph{31st International
  Conference on Machine Learning (ICML)}, ser. PMLR, 2014, pp. 244--252.

\bibitem{barak2016nearly}
B.~Barak, S.~B. Hopkins, J.~Kelner, P.~Kothari, A.~Moitra, and A.~Potechin, ``A
  nearly tight sum-of-squares lower bound for the planted clique problem,'' in
  \emph{57th Symposium on Foundations of Computer Science (FOCS)}.\hskip 1em
  plus 0.5em minus 0.4em\relax IEEE, 2016, pp. 428--437.

\bibitem{ghoshdastidar2014consistency}
D.~Ghoshdastidar and A.~Dukkipati, ``Consistency of spectral partitioning of
  uniform hypergraphs under planted partition model,'' in \emph{Advances in
  Neural Information Processing Systems (NIPS)}, 2014, pp. 397--405.

\bibitem{kim2018stochastic}
C.~Kim, A.~S. Bandeira, and M.~X. Goemans, ``Stochastic block model for
  hypergraphs: Statistical limits and a semidefinite programming approach,''
  \emph{arXiv preprint arXiv:1807.02884}, 2018.

\bibitem{aicher2014learning}
C.~Aicher, A.~Z. Jacobs, and A.~Clauset, ``Learning latent block structure in
  weighted networks,'' \emph{Journal of Complex Networks}, vol.~3, no.~2, pp.
  221--248, 2014.

\bibitem{peixoto2018nonparametric}
T.~P. Peixoto, ``Nonparametric weighted stochastic block models,''
  \emph{Physical Review E}, vol.~97, no.~1, p. 012306, 2018.

\bibitem{dia2016mutual}
J.~Barbier, M.~Dia, N.~Macris, F.~Krzakala, T.~Lesieur, and L.~Zdeborov\'{a},
  ``Mutual information for symmetric rank-one matrix estimation: A proof of the
  replica formula,'' in \emph{Advances in Neural Information Processing Systems
  (NIPS)}, 2016, pp. 424--432.

\bibitem{deshpande2014information}
Y.~Deshpande and A.~Montanari, ``Information-theoretically optimal sparse
  pca,'' in \emph{IEEE International Symposium on Information Theory
  (ISIT)}.\hskip 1em plus 0.5em minus 0.4em\relax IEEE, 2014, pp. 2197--2201.

\bibitem{richard2014statistical}
E.~Richard and A.~Montanari, ``A statistical model for tensor pca,'' in
  \emph{Advances in Neural Information Processing Systems}, 2014, pp.
  2897--2905.

\bibitem{hopkins2015tensor}
S.~B. Hopkins, J.~Shi, and D.~Steurer, ``Tensor principal component analysis
  via sum-of-square proofs,'' in \emph{Conference on Learning Theory (COLT)},
  2015, pp. 956--1006.

\bibitem{montanari2016limitation}
A.~Montanari, D.~Reichman, and O.~Zeitouni, ``On the limitation of spectral
  methods: From the gaussian hidden clique problem to rank one perturbations of
  gaussian tensors,'' \emph{IEEE Transactions on Information Theory}, vol.~63,
  no.~3, pp. 1572--1579, 2016.

\bibitem{jagannath2020statistical}
A.~Jagannath, P.~Lopatto, L.~Miolane \emph{et~al.}, ``Statistical thresholds
  for tensor pca,'' \emph{Annals of Applied Probability}, vol.~30, no.~4, pp.
  1910--1933, 2020.

\bibitem{arous2020algorithmic}
G.~B. Arous, R.~Gheissari, A.~Jagannath \emph{et~al.}, ``Algorithmic thresholds
  for tensor pca,'' \emph{Annals of Probability}, vol.~48, no.~4, pp.
  2052--2087, 2020.

\bibitem{perry2020}
\BIBentryALTinterwordspacing
A.~Perry, A.~S. Wein, and A.~S. Bandeira, ``Statistical limits of spiked tensor
  models,'' \emph{Ann. Inst. H. Poincaré Probab. Statist.}, vol.~56, no.~1,
  pp. 230--264, 02 2020. [Online]. Available:
  \url{https://doi.org/10.1214/19-AIHP960}
\BIBentrySTDinterwordspacing

\bibitem{barbier2020all}
J.~Barbier, N.~Macris, and C.~Rush, ``All-or-nothing statistical and
  computational phase transitions in sparse spiked matrix estimation,''
  \emph{Advances in Neural Information Processing Systems}, vol.~33, 2020.

\bibitem{corinzia2021maximumlikelihood}
L.~Corinzia, P.~Penna, W.~Szpankowski, and J.~M. Buhmann, ``On
  maximum-likelihood estimation in the all-or-nothing regime,'' \emph{arXiv
  preprint arXiv:2101.09994}, 2021.

\bibitem{abbe2018proof}
E.~Abbe and C.~Sandon, ``Proof of the achievability conjectures for the general
  stochastic block model,'' \emph{Communications on Pure and Applied
  Mathematics}, vol.~71, no.~7, pp. 1334--1406, 2018.

\bibitem{deshpande2015finding}
Y.~Deshpande and A.~Montanari, ``Finding hidden cliques of size $\sqrt{N/e}$ in
  nearly linear time,'' \emph{Foundations of Computational Mathematics},
  vol.~15, no.~4, pp. 1069--1128, 2015.

\bibitem{lesieur2017statistical}
T.~Lesieur, L.~Miolane, M.~Lelarge, F.~Krzakala, and L.~Zdeborov{\'a},
  ``Statistical and computational phase transitions in spiked tensor
  estimation,'' in \emph{IEEE International Symposium on Information Theory
  (ISIT)}.\hskip 1em plus 0.5em minus 0.4em\relax IEEE, 2017, pp. 511--515.

\bibitem{biroli2020iron}
G.~Biroli, C.~Cammarota, and F.~Ricci-Tersenghi, ``How to iron out rough
  landscapes and get optimal performances: averaged gradient descent and its
  application to tensor pca,'' \emph{Journal of Physics A: Mathematical and
  Theoretical}, vol.~53, no.~17, p. 174003, 2020.

\bibitem{Kikuchi-hierarchy}
A.~S. {Wein}, A.~{El Alaoui}, and C.~{Moore}, ``The kikuchi hierarchy and
  tensor pca,'' in \emph{2019 IEEE 60th Annual Symposium on Foundations of
  Computer Science (FOCS)}, 2019, pp. 1446--1468.

\bibitem{lopes2018maximum}
M.~E. Lopes, ``On the maximum of dependent gaussian random variables: A sharp
  bound for the lower tail,'' \emph{arXiv preprint arXiv:1809.08539}, 2018.

\bibitem{verdu1994generalizing}
S.~Verd{\'u} \emph{et~al.}, ``Generalizing the fano inequality,'' \emph{IEEE
  Transactions on Information Theory}, vol.~40, no.~4, pp. 1247--1251, 1994.

\bibitem{donoho2009message}
D.~L. Donoho, A.~Maleki, and A.~Montanari, ``Message-passing algorithms for
  compressed sensing,'' \emph{Proceedings of the National Academy of Sciences},
  vol. 106, no.~45, pp. 18\,914--18\,919, 2009.

\bibitem{cover1999elements}
T.~M. Cover, \emph{Elements of information theory}.\hskip 1em plus 0.5em minus
  0.4em\relax John Wiley \& Sons, 1999.

\bibitem{feller2008introduction}
W.~Feller, \emph{An introduction to probability theory and its
  applications}.\hskip 1em plus 0.5em minus 0.4em\relax John Wiley \& Sons,
  2008, vol.~2.

\bibitem{lesieur2017constrained}
T.~Lesieur, F.~Krzakala, and L.~Zdeborov{\'a}, ``Constrained low-rank matrix
  estimation: Phase transitions, approximate message passing and
  applications,'' \emph{Journal of Statistical Mechanics: Theory and
  Experiment}, vol. 2017, no.~7, p. 073403, 2017.

\bibitem{montanari2013statistical}
A.~Montanari, ``Statistical estimation: from denoising to sparse regression and
  hidden cliques,'' \emph{Statistical Physics, Optimization, Inference, and
  Message-Passing Algorithms: Lecture Notes of the Les Houches School of
  Physics: Special Issue}, 2013.

\bibitem{iba1999nishimori}
Y.~Iba, ``{The Nishimori line and Bayesian statistics},'' \emph{Journal of
  Physics A: Mathematical and General}, vol.~32, no.~21, p. 3875, 1999.

\end{thebibliography}

\end{document}